%% file: main.tex
\documentclass{article}

\usepackage{microtype}
\usepackage{graphicx}
\usepackage{booktabs} 
\usepackage{longtable}

\usepackage[hypertexnames=false]{hyperref}

\usepackage[accepted]{icml2024}

\usepackage{amsmath}
\usepackage{amssymb}
\usepackage{mathtools}
\usepackage{amsthm}

\usepackage{microtype}
\usepackage{graphicx}
\usepackage{booktabs} 

\usepackage{algorithm}
\usepackage{algpseudocode}
\usepackage{graphicx} 
\usepackage{todonotes}
\usepackage{xcolor}
\usepackage{dsfont}
\usepackage{nicefrac}
\usepackage{tcolorbox}
\usepackage{tabularx}
\usepackage{pifont}
\usepackage{enumitem}
\usepackage{multirow}
\usepackage{caption}
\usepackage{wrapfig}
\usepackage[export]{adjustbox}
\usepackage[list=true]{subcaption}
\usepackage{colortbl}
\definecolor{lightgray}{RGB}{242, 242, 242}
\usepackage{float}
\usepackage[capitalize,noabbrev]{cleveref}

\input{math_commands.tex}
\input{v_line_patch}
\usepackage{tcolorbox}
\usepackage{xspace}
\DeclareRobustCommand{\ie}{i.e.,\@\xspace}
  
\DeclareRobustCommand{\eg}{e.g.,\@\xspace}
\DeclareRobustCommand{\wrt}{w.r.t.\@\xspace}
\usepackage[capitalize]{cleveref}
\usepackage{tikz}
\usepackage{tikz-cd}
\usetikzlibrary{shapes,backgrounds}
\usepackage{array}
\usepackage{enumitem}

\usepackage{amsthm}
\usepackage{thmtools, thm-restate}
\declaretheorem[numberwithin=section]{thm}

\declaretheorem[sibling=thm]{lemma}

\declaretheorem[]{challenged assumption}
\declaretheorem[]{definition}
\declaretheorem[]{proposition}

\newcounter{relctr} 
\everydisplay\expandafter{\the\everydisplay\setcounter{relctr}{0}} 
\newcommand\labelrel[2]{%
  \begingroup
    \refstepcounter{relctr}%
    \stackrel{\textnormal{(\arabic{relctr})}}{\mathstrut{#1}}%
    \originallabel{#2}%
  \endgroup
}
\AtBeginDocument{\let\originallabel\label} 

\Crefname{figure}{Fig.}{Fig.}
\definecolor{myviolet}{rgb}{0.6, 0.4, 0.8}

\newsavebox{\algGTO}

\DeclareMathOperator*{\argmax}{arg\,max}

\icmltitlerunning{Global Reinforcement Learning: Beyond Linear and Convex Rewards via Submodular Semi-gradient Methods}

\begin{document}

\twocolumn[
\icmltitle{Global Reinforcement Learning: Beyond Linear and Convex Rewards via Submodular Semi-gradient Methods}



\icmlsetsymbol{equal}{*}

\begin{icmlauthorlist}
\icmlauthor{Riccardo De Santi}{equal}
\icmlauthor{Manish Prajapat}{equal}
\icmlauthor{Andreas Krause}{}
\end{icmlauthorlist}

\icmlcorrespondingauthor{Riccardo De Santi}{rdesanti@ethz.ch}

\icmlkeywords{Machine Learning, ICML}

\vskip 0.3in
]



\printAffiliationsAndNotice{\icmlEqualContribution\!. All authors are from ETH Zurich and are affiliated with the ETH AI Center} 

\begin{abstract}
\input{sections/0_abstract}
\end{abstract}
\section{Introduction}
\label{sec:introduction}
\input{sections/1_introduction}
\section{Preliminaries}
\label{sec:preliminaries}
\input{sections/2_preliminaries}
\section{Global Reinforcement Learning (GRL)}
\label{sec:grl_as_co}
\input{sections/3_grl_as_co}
\section{Relation with Convex RL}
\input{sections/table_applications}
\label{sec:general_utilities}
\input{sections/4_general_utilities}
\section{Exploiting Structure in Global RL}
\label{sec:structure}
\input{sections/5_structure}
\section{Semi-gradient Method for GRL}
\label{sec:method}
\input{sections/6_method}
\section{Approximation guarantees and Hardness}
\label{sec:guarantees}
\input{sections/7_guarantees}

\section{Experiments}
\label{sec:experiments}
\input{sections/8_experiments}
\section{Related Work}
\label{sec:related_works}
\input{sections/9_related_works}
\section{Conclusion}
\label{sec:conclusions}
\input{sections/10_conclusions}

\section*{Acknowledgement}
This publication was made possible by the ETH AI Center doctoral fellowship to Riccardo De Santi and Manish Prajapat. We would like to thank Mohammad Reza Karimi and Mojm\'ir Mutn\'y for the
insightful discussions.

The project has received funding from the European Research Council (ERC) under the European
Union’s Horizon 2020 research, innovation program grant agreement No 815943 and the Swiss
National Science Foundation under NCCR Automation grant agreement 51NF40 180545 and NCCR Catalysis grant number 180544. 

\section*{Impact Statement}
This paper presents work whose goal is to advance the field of Reinforcement Learning. There are many potential societal consequences of our work, none which we feel must be specifically highlighted here.

\bibliography{biblio.bib}
\bibliographystyle{icml2024}

\newpage
\begin{appendix}
\onecolumn

\newpage

\section{List of symbols}
\label{sec:notation}
    \begin{longtable}{lll}
        \multicolumn{3}{c}{\underline{\textbf{General Mathematical Objects}}} \\
        \noalign{\vskip 1mm}
             $[N]$ & $\triangleq$ & Set of integers $\{ 1, \ldots, N \}$ \\
             $\Delta(X)$ & $\triangleq$ & Probability simplex over $X$ \\
        \noalign{\vskip 3mm}
        \multicolumn{3}{c}{\underline{\textbf{Controlled Markov Process}}} \\
        \noalign{\vskip 1mm}
             $\cmp$ & $\triangleq$ & Controlled Markov Process (CMP), $\cmp = \langle\Sspace, \Aspace, P, \mu, H\rangle$ \\
             $\Sspace$ & $\triangleq$ & State space \\
             $\Aspace$ & $\triangleq$ & Action space \\
             $P$ & $\triangleq$ & Transition model, $P: \Sspace \times \Aspace \to \Delta(\Sspace)$ \\
             $\mu$ & $\triangleq$ & Initial state distribution, $\mu \in \Delta(\Sspace)$ \\
             $H$ & $\triangleq$ & Horizon \\
             $\T$ & $\triangleq$ & Time steps set, $\T = [H]$ \\
             $s_0$ & $\triangleq$ & Initial state $s_0 \sim \mu$ \\
             $s_t$ & $\triangleq$ & State at time step $t$\\
             $\tau$ & $\triangleq$ & Trajectory $\tau = \{s_0, \ldots, s_{H-1}\}$ \\
             $p_\pi(\tau)$ & $\triangleq$ & Probability of trajectory $\tau$ given a fixed $\cmp$ and $\pi$, see equation \eqref{eq:prob_trajectory} \\
        \noalign{\vskip 3mm}
        \multicolumn{3}{c}{\underline{\textbf{MDP and GMDP}}} \\
        \noalign{\vskip 1mm}
             $\cmp_r$ & $\triangleq$ & Markov Decision Process (MDP), $\cmp_r \coloneqq \langle\Sspace, \Aspace, P, \mu, H, r\rangle$ \\
             $r$ & $\triangleq$ & Scalar reward function $r:\Sspace \to \Reals$ or $r: \Sspace \times \Aspace \to \Reals$ \\
             $\gmdp$ & $\triangleq$ & Global Markov Decision Process, $\gmdp := \langle\Sspace, \Aspace, P, \mu, H, F\rangle$ \\
             $F$ & $\triangleq$ & Global reward function $F: \Gamma_H \to \Reals$ \\
             $\mathcal{C}_{\mathcal{M}}$ & $\triangleq$ & CMP constraint, see definition \ref{def: cmp_constraint}\\
        \noalign{\vskip 3mm}
        \multicolumn{3}{c}{\underline{\textbf{Policies}}} \\
        \noalign{\vskip 1mm}
             $\PimS$ & $\triangleq$ & Class of Markovian stationary policies \\
             $\PimNs$ & $\triangleq$ & Class of Markovian non-stationary policies \\
             $\Pinm$ & $\triangleq$ & Class of non-Markovian policies \\
             $\Hspace_{t}$ & $\triangleq$ & History space until step $t$ \\
             $\pi$ & $\triangleq$ & Markovian policy $\pi : \Sspace \to \Delta(\Aspace)$ or non-Markovian policy $\pi : \Hspace_{t} \to \Delta (\Aspace)$\\
             $\pi_t$ & $\triangleq$ & Markovian Non-stationary policy $\pi_t : \Sspace \times \T \to \Delta(\Aspace)$\\
        \noalign{\vskip 3mm}
        \multicolumn{3}{c}{\underline{\textbf{Combinatorial Structure and Submodular Optimization}}} \\
        \noalign{\vskip 1mm}
            $V$ & $\triangleq$ & Ground set of elements \eg $V = [N]$ \\
            $\F$ & $\triangleq$ & Family of subsets induced by a ground set \eg $\F :=\{0,1\}^V = 2^V$\\
            $F$ & $\triangleq$ & Pseudo-Boolean set-function $F: 2^V \to \Reals$\\
            $k_F$ & $\triangleq$ & Submodular curvature, see \cref{def:submodular_curvature} \\
            $k^F$ & $\triangleq$ & Supermodular curvature, see \cref{def:submodular_curvature} \\
        \noalign{\vskip 3mm}
        \multicolumn{3}{c}{\underline{\textbf{Algorithms: \GTO and \GPO}}} \\
        \noalign{\vskip 1mm}
        $\Rmath(\pi)$ & $\triangleq$ & Non-additive suboptimality gap of policy $\pi$, see equations (\ref{def:suboptimality_gap}, \ref{def:suboptimality_gap_traj})\\
    \end{longtable}
\newpage

\section{Dynamics constraints}
\label{sec:dynamics_constraint}
\input{sections/dynamics_constraint}

\section{Proofs}
\label{sec:proofs}
\input{sections/proofs}
\newpage

\section{Analysis Deterministic Case via Trajectory Optimization}
\label{sec:det_case_appendix}
\input{sections/det_case_appendix}
\newpage

\section{Computational Hardness}
\label{sec:proof_hardness}
\input{sections/proof_hardness}
\newpage

\section{Algorithm}
\label{sec:appx_algorithm}
\input{sections/appx_algorithm}
\newpage

\section{Experiments Details}
\label{sec:experiments_details}
\input{sections/experiments_details}
\end{appendix}
\end{document}

%% file: math_commands.tex
\usepackage{amsmath,amsfonts,bm}

\newcommand{\mdp}{\mathcal{M}}
\newcommand{\Aspace}{\mathcal{A}}
\newcommand{\Sspace}{\mathcal{S}}
\newcommand{\Hspace}{\mathcal{H}}
\newcommand{\Reals}{\mathbb{R}}
\newcommand{\indicator}{\mathbb{I}}
\newcommand{\cmp}{\mathcal{M}}
\newcommand{\gmdp}{\cmp_{F}}
\newcommand{\F}{f}
\newcommand{\J}{\mathcal{J}}
\newcommand{\C}{\mathcal{C}}
\newcommand{\V}{\mathcal{V}}
\DeclareMathOperator*{\EV}{\mathbb{E}}

\DeclareMathOperator*{\Pinm}{\Pi_{NM}}
\DeclareMathOperator*{\PimNs}{\Pi_{M}^{NS}}
\DeclareMathOperator*{\PimS}{\Pi_{M}^S}

\newcommand{\GTO}{\textsc{\small{GTO}}\xspace}
\newcommand{\GPO}{\textsc{\small{GPO}}\xspace}
\newcommand{\GRL}{\textsc{\small{GRL}}\xspace}
\newcommand{\BP}{\textsc{\small{BP}}\xspace}
\newcommand{\RL}{\textsc{\small{RL}}\xspace}
\newcommand{\MDP}{\textsc{\small{MDP}}\xspace}
\newcommand{\CMP}{CMP\xspace}
\newcommand{\GMDP}{\textsc{\small{GMDP}}\xspace}
\newcommand{\mdpSolver}{\textsc{\small{MdpSolver}}\xspace}
\newcommand{\R}{\mathbb{R}}
\newcommand{\T}{\mathcal{T}}
\newcommand{\Rmath}{\mathcal{R}}

\newcommand{\mypar}[1]{\textbf{#1.}}



%% file: sections/0_abstract.tex
\looseness -1 In classic Reinforcement Learning (RL), the agent maximizes an additive objective of the visited states, \eg a value function. Unfortunately, objectives of this type cannot model many real-world applications such as experiment design, exploration, imitation learning, and risk-averse RL to name a few. This is due to the fact that additive objectives disregard interactions between states that are crucial for certain tasks. To tackle this problem, we introduce \emph{Global} RL (GRL), where rewards are \emph{globally} defined over trajectories instead of \emph{locally} over states. Global rewards can capture \emph{negative interactions} among states, \eg in exploration, via submodularity,  \emph{positive interactions}, \eg synergetic effects, via supermodularity, while mixed interactions via combinations of them. By exploiting ideas from submodular optimization, we propose a novel algorithmic scheme that converts any GRL problem to a sequence of classic RL problems and solves it efficiently with curvature-dependent approximation guarantees. We also provide hardness of approximation results and empirically demonstrate the effectiveness of our method on several GRL instances.

%% file: sections/1_introduction.tex
\begin{figure}[t]
    \centering
    \includegraphics[width=0.47\textwidth]{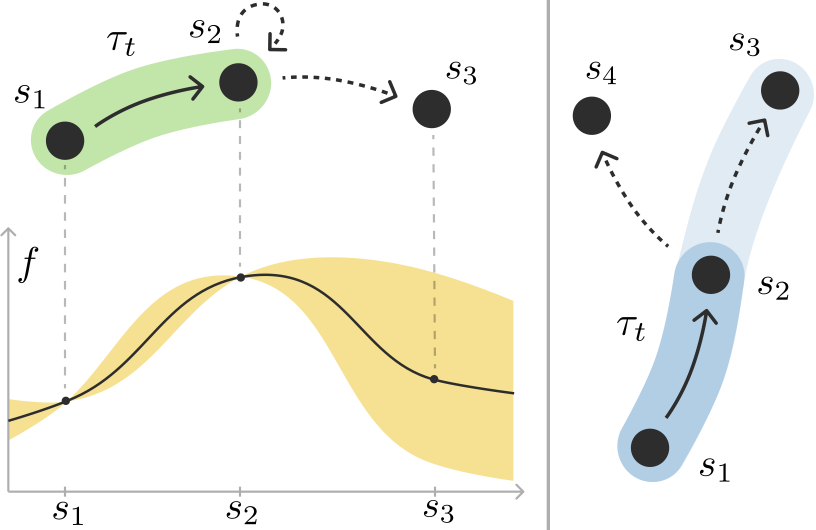}
    \caption{\looseness -1 The agent has visited trajectory $\tau_t$ and must select the next state. On the left, the agent aims to estimate an unknown state function $f$: re-visiting $s_2$ leads to a negative interaction since the information gain has diminishing returns. On the right, the agent seeks a trajectory, \ie ordered set of atoms, maximizing synergies seen as positive interactions among certain combinations of atoms \eg adding $s_3$ to $\tau_t=\{s_1,s_2\}$ leads to a synergetic effect.}
    \label{fig:illustrative_applications}
\end{figure}

\looseness -1 Classic Reinforcement Learning (RL)~\cite{puterman2014markov} represents the value of a trajectory as a sum of \emph{local} rewards over its states(-actions). This fact allows us to exploit Bellman's optimality principle and therefore can find optimal policies using efficient algorithms inspired by dynamic programming~\cite{bellman1966dynamic}. Unfortunately, additive objectives cannot properly capture a multitude of real-world tasks including pure exploration~\cite{hazan2019maxent, mutti2022importance, liu2021behavior}, function-estimation or experimental design~\cite{mutny2023active, tarbouriech2019active, tarbouriech2020active}, imitation learning or distribution matching~\cite{abbeel2004apprenticeship, state_marginal, ghasemipour2020divergence}, risk-averse RL~\cite{garcia2015comprehensive, mutti2022challenging}, diverse skill discovery~\cite{eysenbach2018diversity, campos2020explore, he2022wasserstein}, and constrained RL~\cite{brantley2020constrained, qin2021density}. In these cases, the interactions between states play a fundamental role in determining the performance of a trajectory. As an example, consider the case of experiment design over a Markov chain~\cite{mutny2023active, tarbouriech2019active} in \cref{fig:illustrative_applications} (left), where an agent aims to minimize uncertainty about an unknown function of the states, by observing a noisy sample of it when visiting a state. Intuitively, the new information gained by observing the function value at a state depends on how often that state has been previously visited: the more it has been visited, the less new information is gained. This phenomenon of \emph{negative interactions} between states cannot be captured by additive objectives that simply sum over (fixed) local state rewards \citep{prajapat2023submodular}. Consider a second example (see \cref{fig:illustrative_applications}, right) where an RL agent is used to design molecules~\cite{thiede2022curiosity} by selecting a set of atoms, where each atom is represented as a state, via a trajectory of fixed length. In this case, \emph{positive interactions} between states, \eg synergies, can capture the positive effect of including within the trajectory certain combinations of atoms. But once again, additive objectives used in classic RL can only assign a fixed local reward to each atom thus disregarding the interactions among them.

To tackle this problem, we introduce \emph{Global} RL (GRL), where an agent aims to compute a policy maximizing rewards that are \emph{globally} defined over trajectories rather than \emph{locally} over states. This makes it possible to capture non-additivity and interactions among states even in finite-sample processes.
Then, we formally show how GRL can be interpreted as a specific combinatorial optimization problem (Section \ref{sec:grl_as_co}), and study its relation with Convex RL~\cite{hazan2019maxent, zahavy2021reward} (Section \ref{sec:general_utilities}), an alternative framework to deal with non-additive objectives.
It is easy to see that Global RL is a hard problem in general, thus the rest of the work aims to answer the question: 
\begin{center}
    \emph{When and how can we efficiently and\\ approximately solve the Global RL problem?}
\end{center}
Towards answering this question, we extend discrete semi-gradient methods from submodular optimization~\cite{iyer2013fast} to design a meta-algorithm that converts a GRL problem into a sequence of classic RL planning problems. 
Then, we identify several structural properties of global rewards, leading to approximation guarantees for our algorithmic scheme. 
Among these, \emph{submodular} ~\cite{lovasz1983submodular, krause2014submodular, krause2011submodularity} global rewards capture negative interactions between states, \emph{supermodular}~\cite{gallo1989supermodular, billionnet1985maximizing} global rewards capture positive interactions, and monotone \emph{suBmodular-suPermodular} (BP)~\cite{bai2018greed, ji2019stochastic} global rewards capture mixed interactions.
We show that these reward structures model a wide range of applications that cannot be expressed via local rewards, including maximum entropy exploration~\cite{hazan2019maxent, mutti2022importance}, informative path planning~\cite{prajapat2023submodular}, experiment design~\cite{mutny2023active}, and synergetic trajectory selection among others. 
Furthermore, we use these structures to study the computational hardness of approximation results and perform a thorough experimental evaluation of the proposed methods (Section \ref{sec:experiments}) in the context of experimental design, optimization of design processes, and safe exploration.
To sum up, in this work we present the following contributions:
\begin{itemize}[noitemsep,topsep=0pt,parsep=0pt,partopsep=0pt,leftmargin=*]
    \item The notion of \emph{Global MDP} and the \emph{Global RL} (GRL) problem, which generalizes RL to non-additive objectives.
    \item A general algorithmic scheme to solve GRL by converting it to a sequence of classic MDPs via submodular semi-gradient methods (\cref{sec:method}).
    \item \looseness -1 Approximation guarantees for the proposed algorithms via the notion of \emph{curvature} that explicitly connect the degree of non-additivity of a global reward with the approximation ratio (\cref{sec:guarantees}).
    \item A computational hardness result for GRL, thereby ruling out the possibility of achieving better approximation ratios (\cref{sec:guarantees}).
    \item An extensive experimental evaluation of the proposed algorithms on a wide range of applications (Section \ref{sec:experiments}).
\end{itemize}

%% file: sections/2_preliminaries.tex
We denote with $[N]$ a set of integers $\{ 1, \ldots, N \}.$ Let $X$ be a set, $\Delta(X)$ is the probability simplex over $X$.

\looseness -1 \mypar{Controlled Markov Process (\CMP)} An episodic \CMP~\cite{puterman2014markov} is a tuple $\cmp \!:=\! \langle\Sspace, \Aspace, P, \mu, H\rangle$, where $\Sspace$ is a discrete state space, $\Aspace$ is a discrete action space, $P\!:\! \Sspace \times \Aspace \!\to \!\Delta (\Sspace)$ is the transition model, where $P(s'|s,a)$ is the probability of reaching state $s'$ by taking action $a$ in state $s$. Meanwhile, $\mu \in \Delta (\Sspace)$ is the initial state distribution, $H$ is the horizon of an episode, and we define $\T = [H]$.
In each episode, the agent observes an initial state $s_0 \sim \mu$, selects an action $a_0$, and transitions to $s_1 \sim P(\cdot|s_0, a_0)$. This interaction process is repeated until the episode ends. 

\looseness -1 \mypar{Markov Decision Process (MDP)} If a \CMP \mbox{$\cmp := \langle\Sspace, \Aspace, P, \mu, H\rangle$} is augmented with a scalar reward function \mbox{$r:\Sspace \to \Reals$} or \mbox{$r: \Sspace \times \Aspace \to \Reals$} then we obtain the MDP \mbox{$\cmp_r \coloneqq \langle\Sspace, \Aspace, P, \mu, H, r\rangle$}.

\looseness -1 \mypar{Policies} 
A \emph{policy} encodes the behavior of an agent interacting with a \CMP. A non-Markovian policy \mbox{$\pi \in \Pinm$} is a function \mbox{$\pi : \Hspace_{t} \to \Delta (\Aspace)$}, where $\Hspace_{t}$ denotes the set of all histories, \ie states visited in the past, up to length $t$. A Markovian non-stationary policy $\pi_t \in \PimNs$ is a function \mbox{$\pi_t : \Sspace \times \T \to \Delta(\Aspace)$}, while a Markovian stationary policy \mbox{$\pi \in \PimS$} is a function \mbox{$\pi : \Sspace \to \Delta(\Aspace)$} independent of the time-step.
One can notice that $\PimS \subseteq \PimNs \subseteq \Pinm$. Moreover, we denote with $\Pi$ an arbitrary policy class.

\mypar{Set Functions}
We denote with the term \emph{ground set} a set $\V$ inducing the \emph{family} of subsets $\F :=\{0,1\}^\V = 2^{\V}$. Notice that a function $F: 2^\V \to \Reals$ takes subsets of the ground set $\V$ as input and outputs scalars, formally $F: X \mapsto r$ with $X\subseteq \V$ and $r \in \Reals$. Moreover, every set $X \subseteq \V$ can be represented as a binary vector $x \in \{0,1\}^\V$ with entries given by $x(i) = \mathbb{I}_{i \in X}$.

%% file: sections/3_grl_as_co.tex
\looseness -1 In this section, we formulate the Global RL (GRL) problem and shed light on its connection with combinatorial optimization, which will be essential to design efficient approximate algorithms. Towards this goal, we first introduce the concept of Global MDP (GMDP), which generalizes the notion of MDP to the case of general non-additive reward functions.
\begin{restatable}[Global Markov Decision Process]{definition}{globalMDPdef}
\label{def:global_mdp_def}
    Consider a CMP $\cmp := \langle\Sspace, \Aspace, P, \mu, H\rangle$ and a global reward function $F: 2^{\Sspace \times \T} \to \Reals$ mapping trajectories  to scalar returns. We define a \emph{Global Markov Decision Process} (GMDP) as the tuple $\gmdp := \langle\Sspace, \Aspace, P, \mu, H, F\rangle$.
\end{restatable}
\looseness -1 We denote as \emph{local} a reward function, \eg $r: \Sspace \to \Reals$, that assigns a scalar reward to each state(-action) and as \emph{global} a reward function $F\!:\!2^{\Sspace \times \T} \mapsto \Reals$ that assigns a scalar reward to each trajectory $\tau \! \coloneqq \{ (s_t,t)_{t=0}^{H-1}\} \subseteq \Sspace \times \T$.
Given a GMDP, we can state the GRL problem as follows.
\begin{tcolorbox}[colframe=white!, top=2pt,left=2pt,right=2pt,bottom=2pt]
\center \textbf{Global Reinforcement Learning}
\begin{equation}
\label{eq:global_reinforcement_learning}
    \max_{\pi \in \Pi} \;\; \J(\pi) \coloneqq \EV_{\tau \sim p_\pi} \big[ F(\tau) \big]
\end{equation}
\end{tcolorbox}
Hereby, $p_\pi$ is the distribution over trajectories induced by policy $\pi$, formally: 
\begin{equation*}
\label{eq:prob_trajectory}
    p_\pi(\tau) = \mu(s_0)\prod_{t=0}^{H-1} \pi_t(a_t | s_t) P(s_{t+1} | s_t, a_t)
\end{equation*}
Particularly, in this work we focus on the setting of known transition model $P$ and global rewards $F$. 
\subsection{GRL as a Subset Selection Problem}
Given a family $\mathcal{F} = 2^\V$ induced by a ground set $\V$, a representative problem in Combinatorial Optimization (CO) is the \emph{subset selection problem}~\cite{das2011submodular}, where one aims to find an optimum of a set-function $F: \mathcal{F} \to \Reals$ while constrained to a sub-family $\mathcal{C} \subseteq \mathcal{F}$, as in Equation \ref{eq:subset_selection_problem}.
\begin{tcolorbox}[colframe=white!, top=2pt,left=2pt,right=2pt,bottom=2pt]
\center \textbf{Subset Selection Problem}
\begin{equation}
\label{eq:subset_selection_problem}
    \max_{X \in \C} \;\;  F(X) 
\end{equation}
\end{tcolorbox}

Towards interpreting GRL as a CO problem, we define a path constraint denoted as \emph{dynamics constraint}, and indicate it with $\mathcal{C}_{\mathcal{M}}$~\cite{blum2007approximation}. Given a CMP $\cmp$, $\mathcal{C}_{\mathcal{M}}$ intuitively represents the set of admissible trajectories according to the dynamics $P$. A formal construction of $\mathcal{C}_{\mathcal{M}}$ based on the time-extended CMP of $\cmp$ is presented in Appendix \ref{sec:dynamics_constraint}.\footnote{A time-extended CMP has state space $\V \coloneqq \Sspace \times \T$.}

Given the notion of dynamics constraint, we can define the trajectory-optimization version of GRL \eqref{eq:global_reinforcement_learning} for {\em deterministic} GMDPs, i.e., fixed initial state and deterministic transitions, as the following subset selection problem:
\begin{tcolorbox}[colframe=white!, top=2pt,left=2pt,right=2pt,bottom=2pt]
\center \textbf{Global RL: Trajectory-Optimization} 
\begin{equation}
\label{eq:global_reinforcement_learning_traj}
    \max_{\tau \in \C_{\cmp}} \;\;  F(\tau) 
\end{equation}
\end{tcolorbox}
\looseness -1 Notice that this problem formulation is sufficient to find optimal policies in deterministic GMDPs as in this case there exists an optimal deterministic policy~\cite{prajapat2023submodular}, which can be interpreted as a trajectory.
Moreover, this formulation can be straightforwardly extended to the general problem in Equation \eqref{eq:global_reinforcement_learning} by replacing $F(\tau)$ by its expectation according to $p_\pi$. Nonetheless, due to the natural analogy between trajectories and sets, rather than between policies and distributions over sets, we will first introduce novel concepts for the trajectory-optimization version of GRL \eqref{eq:global_reinforcement_learning_traj}, and then extend them to the general GRL problem \eqref{eq:global_reinforcement_learning}.

\looseness -1 Crucially, solving GRL (\cref{eq:global_reinforcement_learning,eq:global_reinforcement_learning_traj}) is in general inapproximable, since even for a restricted class of global rewards, it is intractable up to any constant factors as shown by ~\citet[Theorem 1]{prajapat2023submodular}. Nonetheless, in Section \ref{sec:structure}, by leveraging the CO viewpoint introduced within this section, we identify structural properties of $F$ that are both common in practice and lead to efficient approximate algorithms.

%% file: sections/table_applications.tex
\begin{table*}[!htbp]
\setlength{\tabcolsep}{4pt}
\renewcommand{\arraystretch}{1.1}
\caption{\looseness -1 Applications of Global RL with (from top) submodular, supermodular, BP, and arbitrary global rewards.}
\vspace{-1mm}
\label{table:GRLapplications_all}
\begin{sc}\begin{small}
\resizebox{\textwidth}{!}{%
\begin{tabular}{c c c c}
\toprule
Application & Set function $F(\tau)$ & Details \\
\midrule
State entropy exploration & $
        \frac{-1}{|\tau|} \sum_{s \in \mathcal{S}} \mathbb{I}_{(s,\cdot)\in\tau} \log \frac{|\{t: (s,t)\in \tau \} |}{|\tau|} 
    $  & \\[0.1em]
\rowcolor{lightgray} Goal reachability & $\indicator\{|\tau \cap S_g|>0\}$ & \\
Coverage & $|\bigcup_{(s,t)\in \tau} D^s|$ &\\ 
\rowcolor{lightgray}Bounded Curvature Coverage & $\sum_{s \in \Sspace} \mathbb{I}_{C(\tau,s)>0}\cdot [1-\alpha(C(\tau,s)-1)]$ & $0 \leq \alpha \leq 1$\\
Informative path planning  &  $g(\bigcup_{(s,t) \in \tau} D^s)$  & $g(V) = \sum_{v \in V} \rho(v)$\\
\rowcolor{lightgray}D-Optimal Experimental Design & $I(y_\tau;f) = H(y_\tau) - H(y_\tau | f)$ & \\
Neighbours coverage in space-time & $\sum_{v \in V} \max \{ \alpha, \min\{ |S \cap S_v|, 1 \} \}$ & $0 \leq \alpha \leq 1$\\
\rowcolor{lightgray}Coverage of time-varying processes & $|\bigcup_{v\in \tau} D^v|$ & \\
\midrule
Goal Completion & $\indicator \{\tau \cap S_g = S_g\}$ & $S_g \subseteq V $\\
\rowcolor{lightgray}Automatic Task Selection & $\sum_{i=1}^N R_i(\tau)^\beta$ & $R_i(\tau) = \sum_{s \in \tau} r_i(s), \beta > 1 $\\
Synergical Trajectory Selection & $\sum_{i=1}^K |\tau \cap S_i|^\beta$ & $S_i \subseteq V, \beta > 1$\\
\midrule
\rowcolor{lightgray}Diverse and Synergical Trajectory Selection & $ |\bigcup_{s\in \tau} D^s| +\sum_{i=1}^K |\tau \cap S_i|^\beta$ & $\beta \geq 1$\\ [0.1em] 
\midrule
Safe Reward Maximization  & $R(\tau) + C \cdot \indicator \{\tau \cap S_{u} = 0\} $ & R \text{additive}\\
\rowcolor{lightgray}Submodular + Safety & $Q(\tau) + C \cdot \indicator \{\tau \cap S_{u} = 0\}$ & Q submodular\\
\bottomrule
\end{tabular}
}\end{small}
\end{sc}
\vspace{-0.4cm}
\end{table*}

%% file: sections/4_general_utilities.tex
\looseness -1 Before further discussing the GRL problem, we establish a connection with the area of \emph{General Utilities RL} (GURL)~\cite{zahavy2021reward, geist2021concave, zhang2020variational, barakat2023reinforcement}, which offers an alternative way to tackle non-additive objectives. As stated in \cref{eq:general_utilities_obj},
\begin{equation}
\label{eq:general_utilities_obj}
    \max_{d^{\pi} : \pi \in \PimS} \F(d^{\pi}) 
\end{equation}
the GURL objective is to find a policy $\pi \in \PimS$ inducing an optimal state(-action) distribution \wrt a given functional $\F:\Delta(\Sspace \times \Aspace) \to \Reals $.
If $\F$ is convex (concave) in $d^\pi$, the problem in equation \eqref{eq:general_utilities_obj} is referred to as \emph{Convex} (\emph{Concave}) RL (CRL)~\cite{geist2021concave, zhang2020variational, zahavy2021reward}. In this case, the problem can be efficiently solved via standard constrained convex optimization schemes~\cite{hazan2019maxent}, 
but unfortunately, it is characterized by a fundamental modelling limitation.
\subsection{Fundamental Limitation of Convex RL}
\looseness -1 Recently, it has been shown \cite{mutti2022challenging, mutti2023convex, mutti2022importance} that both theoretically and experimentally, an optimal policy \wrt the CRL objective \eqref{eq:general_utilities_obj} can perform arbitrarily poorly when released in an environment for a finite amount of interactions, which is unfortunately the case in most real-world applications. This phenomenon is due to the fact that CRL in \cref{eq:general_utilities_obj} optimizes asymptotic distributions rather than their empirical counterparts. 
To tackle this problem,~\citet{mutti2023convex} propose \emph{Single Trial} Convex RL (ST-CRL) (Equation \ref{eq:single_trial_general_utilities_rl}), which captures the finite-samples nature of the problem by optimizing the expected performance of an empirical distribution $d \in \Delta(\Sspace \times \Aspace)$ induced by the interaction of a policy $\pi$ with the environment for a finite number of steps.
\begin{equation}
\label{eq:single_trial_general_utilities_rl}
    \max_{\pi \in \Pinm} \;\; \EV_{d \sim p^\pi} \big[ \F(d) \big] 
\end{equation}
\looseness -1 Problem \eqref{eq:single_trial_general_utilities_rl}  does not suffer from the aforementioned issue; however, it is intractable, and developing algorithms that approximately solve Problem \eqref{eq:single_trial_general_utilities_rl} is still an open problem~\cite{mutti2022importance}. Notably, also GRL~\eqref{eq:global_reinforcement_learning}  overcomes the modelling limitation of Convex RL by directly optimizing a set function defined over trajectories of finite length.
Moreover, interestingly, we show that any ST-CRL problem \eqref{eq:single_trial_general_utilities_rl} can be rewritten as a Global RL problem \eqref{eq:global_reinforcement_learning}.
\begin{restatable}[Single Trial Convex RL $\subseteq$ Global RL]{proposition}{equivalence}
    \label{proposition:equivalence}
    Given an instance $\mathcal{I^+}$ of ST-CRL
    it is possible to reduce it to an instance  $\mathcal{I_+}$ of GRL \eqref{eq:global_reinforcement_learning}.
\end{restatable}

\looseness -1 The proof can be found in appendix \ref{sec:proofs}. Crucially, in ST-CRL \eqref{eq:single_trial_general_utilities_rl} convexity is lost due to the empirical distributions constraints set, and alternative structural properties useful for optimization are not known yet. In contrast, GRL correctly captures the underlying combinatorial nature of the problem. This fact makes it possible to leverage structural assumptions for efficient approximate optimization common in a wide variety of real-world problems, including typical CRL applications, as shown in the next sections.

%% file: sections/5_structure.tex
\looseness -1 As previously mentioned, solving the GRL problem is computationally intractable even up to constant factor approximations~\citep{prajapat2023submodular,chekuri_rg}. Nonetheless, in this section, we introduce two fundamental components of the global rewards set function $F$, namely \emph{submodular} and \emph{supermodular} rewards, which we leverage to approximately solve \GRL efficiently. In the following, we show that these properties offer an intuitive characterization and can be used to model a variety of applications by decomposing global rewards into two fundamental components.

\begin{definition}[Submodular rewards] \label{def:submodularity} \looseness -1 A global reward function \mbox{$Q:2^{\V} \to \R$} is called \emph{submodular} if for every \mbox{$\tau_A \subseteq \tau_B \subseteq \V$} and \mbox{$v \in \V \backslash \tau_B$} it holds that \mbox{$Q(\{v\}|\tau_A) \geq Q(\{v\} | \tau_B)$}, where the marginal gain (discrete derivative) \mbox{$Q(\{v\}|\tau_A) \coloneqq Q(\{v\} \cup \tau_A) - Q(\tau_A)$}. 
\end{definition}
\looseness -1 Thus, submodularity naturally captures a diminishing return property, i.e., the marginal gain of adding a state $v$ to a smaller trajectory $\tau_A$ is higher as compared to adding it's super set $\tau_B$. This denotes \emph{negative interaction} between states i.e. similar states are discouraged and thus maximizing such functions will encode diversity in the resulting trajectory. 

\begin{definition}[Supermodular rewards] \label{def:supermodularity} \looseness -1 A global reward function $G:2^{\V} \to \R$ is \emph{supermodular} if for every $\tau_A \subseteq \tau_B \subseteq \V$ and $v \in \V \backslash \tau_B$ it holds that $G(\{v\}|\tau_A) \leq G(\{v\} | \tau_B)$. 
\end{definition}

\looseness -1 Therefore, contrary to submodular, maximizing supermodular rewards encodes complementarity in the resulting trajectory, i.e., having similar states complement each other and results in larger gains -- \emph{positive interaction}. Next, combining both we define another reward class called monotone suBmodular suPermodular (\BP) rewards \citep{bai2018greed}:
\begin{definition}[\BP rewards]
\looseness -1 A global reward function $F: 2^\V \to \R$ admits a \BP decomposition if there exists a submodular reward function $Q$ and a supermodular reward function $G$ both normalized $(Q(\emptyset)=G(\emptyset)=0)$ and monotonic non-decreasing $(Q(\{v\}|\tau_A) \geq 0, G(\{v\}|\tau_A) \geq 0, \forall v \in \V, \tau_A \subseteq \V$ such that $F(\tau_A) = Q(\tau_A) + G(\tau_A) ~\forall \tau_A \subseteq \V$.
\end{definition}
\looseness -1 Thus, BP rewards enable us to capture applications involving both positive and negative interactions. Interestingly, any arbitrary global reward $F: 2^\V \to \R$ can be decomposed as the sum of submodular and supermodular rewards, i.e., $F(\tau_A) = Q(\tau_A) + G(\tau_A)$ as shown for set functions by \citet[Lemma 4]{narasimhan2012submodular}. Moreover, under certain conditions, this decomposition can be computed in polynomial time \citep[c.f. Lemma 3.2]{iyer2012algorithms}. Crucially, in \cref{sec:structure} we present an algorithmic scheme that does not make any assumption (e.g., monotonicity) on the submodular or supermodular component, hence it is applicable to any GMDPs with arbitrary global rewards as long as the decomposition is available.

\mypar{Examples of global rewards} In \cref{table:GRLapplications_all}, we show a wide range of relevant applications that can be modelled with global rewards. Applications such as state entropy exploration, D-optimal experiment design, where the objective is maximized by visiting a diverse set of states, are captured using submodular rewards. In contrast, the supermodular component captures positive interactions among states, \eg certain synergies among a specific set of states. Combining both components (\BP rewards) becomes particularly relevant when addressing intricate processes that require both, e.g., diverse and synergical trajectory selection. 
Moreover, we can model constraints using arbitrary (e.g., non-monotone \BP) global rewards and thus extend to applications involving policy learning under safety-critical conditions.
Furthermore, since our rewards are defined on state time pairs, we can also model time-varying processes. 

%% file: sections/6_method.tex
\looseness -1 In this section, we present our novel algorithmic scheme to solve the \GRL Problem. We begin by explaining the notion of subdifferentials, which are used to obtain semi-gradients of global rewards. Using the semi-gradients, we first present the algorithm for deterministic GMDPs (\ref{eq:global_reinforcement_learning_traj}) in order to build intuition, and then extend it to the general Problem \eqref{eq:global_reinforcement_learning}.

\looseness -1 \mypar{Subdifferentials} 
Given a non-additive set function $F$, similar to convex functions, we can define the subdifferential $\partial_F(X)$ at a set $X \subseteq V$ as shown in \citet{iyer2015polyhedral}:
\begin{equation*}
\partial_F(X) \! := \! \{x \in \R^n \!\!: \!F(Y) \!\geq\! F(X) + x(Y) - x(X),\!  \forall  Y \!\subseteq\! V\}
\end{equation*}
Here, $x(A) \coloneqq \sum_{v \in A} x(v)$ is a modular function over $V$. Although the polyhedron $\partial_F(X)$ can be defined for any set function, it can be characterized efficiently for submodular functions. We denote a subgradient of $F$ at $X$ as $h_X \in \partial_F(X)$. For submodular $F$, the extreme points of the polyhedron can be characterized as follows:
   
\looseness -1 First, we define a permutation $\sigma: [|V|] \to V$ that re-orders the elements of $V$ such that the elements of $X$ are assigned to the first $|X|$ positions of $V$ (i.e., $\sigma(i) \in X \iff i \leq |X|$), whereas the remaining elements are assigned arbitrarily. We define the set $S^\sigma_0 := \emptyset$ and $S^\sigma_i := \{\sigma(1), \ldots, \sigma(i)\}$. It is important to note that as a consequence, we have $S_{|X|} = X$. With this, we obtain the extreme point $h^\sigma_{X}$ of the polyhedron $\partial_F(X)$ with entries $h^\sigma_{X}(\sigma(i)) = F(S^\sigma_i) - F(S^\sigma_{i-1})$.
Now, given any subgradient $h_X$, we define a modular function,
\begin{equation*}
    m^\sigma_X(Y) := F(X) + h_X(Y) - h_X(X), \label{eq:lb_state}
\end{equation*}
which is defined $\forall Y \subseteq V$ and is a tight lower bound of $F$, i.e., $m_X(X) = F(X)$ with $m_X(Y) \leq F(Y), \forall Y \subseteq V$, see \cref{prop:tight_LB} for details. Notably, if we choose subgradient $h_X$ to be the extreme point of $\partial_F(X)$, then the modular function simplifies to $m_X(Y) = h_X(Y)$, which is a marginal gain of $F$ evaluated with respect to some permutation.

\looseness -1 Analogously, for supermodular functions we define the following modular lower bound $\forall Y \subseteq V$ about the set $X$, as shown in \citet{bai2018greed, iyer2012submodular}:
\begin{equation*}
    m_{X}(Y)\!:= F(X) - \!\!\! \sum\limits_{j \in X \backslash Y} \!\! F(j \mid X \backslash j) + \!\!\! \sum\limits_{j \in Y \backslash X} \!\! F(j \mid \emptyset). 
\end{equation*}
\looseness -1 Then, $m_{X}(Y)\! \leq \!F(Y), \forall Y \!\subseteq \!V$ and $m_{X}(X) \!= \!F(X)$. Furthermore, by adding the tight modular lower bounds for submodular and supermodular functions, we obtain tight modular lower bounds for global rewards \citep{bai2018greed}.

\begin{figure}[t]
    \centering
    \includegraphics[width=0.40\textwidth]{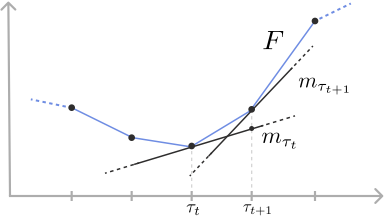}
    \caption{\looseness -1 In \GTO, at any step $t$, we construct $m_{\tau_t}$, a tight modular lower bound about $\tau_t$ and optimize the resulting classic MDP, which results in an improved trajectory $\tau_{t+1}$.}
    \label{fig:lower_bounds}
\end{figure}

\mypar{Global Trajectory Optimization (\GTO)} \looseness -1 Naturally, two key questions emerge: How can we use these modular lower bounds to solve the GRL problem? And, how do we make sure the dynamics constraints are satisfied?
Intuitively, a modular lower bound of the global reward function represents nothing but an additive reward. We recursively approximate the global rewards function about the current trajectory $\tau$ with additive rewards (as shown in \cref{fig:lower_bounds}), which results in a sequence of classic RL planning problems. These problems can be optimized efficiently through standard \RL techniques while ensuring admissible trajectories, i.e., satisfying the dynamics constraints. 
In simpler terms, our algorithmic scheme is a meta-algorithm that approximates GMDP with a sequence of local MDPs which then are solved with existing MDP solvers. The lower bounds being tight at the current trajectory ensures monotonic improvement across the iterations until convergence. Since $F$ is defined on $\V \coloneqq \Sspace \times T$, our modular lower bounds are defined on $\V$.  To construct a modular lower bound tight at a trajectory $\tau$, we define a permutation $\sigma_\tau$ induced by $\tau$ as follows, 
\begin{equation}
    \sigma_\tau = \{ \tau; \V \backslash \tau\} \label{eq:tight_permutation}
\end{equation}
In the described permutation, the elements within the set $\V \backslash \tau$ or the set $\tau$ can be arbitrarily permuted among themselves, while still resulting in a valid tight modular lower bound (\cref{prop:tight_LB}). This flexibility can also be exploited to incorporate prior knowledge for learning better policies. We refer the reader to \cref{sec:alternative_LB} for further details on modular lower bounds.

\looseness -1 The resulting algorithm is summarized in \cref{alg:global_traj_optim}.
We start with an arbitrary trajectory $\tau_1$. The algorithm has two steps: \textit{i)} computes a modular lower bound of the \BP function as described above to obtain additive rewards (classic MDP) about the current trajectory $\tau_t$ (Line \ref{alg:GTO:linearize_MDP}) and \textit{ii)} solves the resulting classic \MDP using existing \mdpSolver tools, e.g., value iteration, which guarantees an optimal solution corresponding to the rewards $m_{\tau_{t}}$ (Line \ref{alg:GTO:solve_MDP}). On solving, the new trajectory will achieve a monotonic improvement in objective value (\cref{lem:monotonic_improvement}). The algorithm alternates between the two steps until convergence, i.e., the objective does not improve anymore.

\savebox{\algGTO}{%
\begin{minipage}[t]{.48\textwidth}
\begin{algorithm}[H]
    \caption{Global Trajectory Optimization (\GTO)}
    \label{alg:global_traj_optim}
        \begin{algorithmic}[1]
        \State{\textbf{Input:} Deterministic \GMDP, $\tau_1$}
        \For{$t=1, 2, \hdots$}
            \State{\!\!\!$m_{\tau_t} \leftarrow$ Compute lower bound of F around $\tau_t$ \label{alg:GTO:linearize_MDP}}
            \State{\!\!\!$\tau_{t+1}\!\leftarrow \mdpSolver(\mdp=\!\langle \Sspace, \Aspace, P, s_0, H, m_{\tau_t} \rangle)$  \label{alg:GTO:solve_MDP} }
        \EndFor
        \State{Return $\tau_{t}$}
        \end{algorithmic}
\end{algorithm}
\end{minipage}}
\begin{figure}
\usebox{\algGTO}
\end{figure}

\looseness -1 \mypar{Global Policy Optimization (\GPO)} For stochastic GMDPs, we solve the problem in policy space. For this, we generalize the concept of lower bound about a trajectory to distribution over trajectories generated by a policy. The corresponding modular lower bound about a policy, $m^{E}_{\pi} \in \R^V$ is defined as the expectation of the modular rewards about the trajectories induced by the policy $\pi$, i.e.,
\begin{align}
    m^{E}_{\pi} \coloneqq \EV_{\tau \sim \pi} [m_{\tau}]. \label{eq: stochastic_lb}
\end{align}
This can be estimated, for instance, using Monte Carlo samples. Next, we define the linearization of the \GRL objective $\J(\pi)$ about a policy $\pi'$ and evaluating it \eqref{eq: stochastic_lb} at a policy $\pi$ as,
\begin{align}
    m^{E}_{\pi'}(\pi) \coloneqq \EV_{\tau \sim \pi} \EV_{\tau' \sim \pi'} [m_{\tau'}(\tau)].
\end{align}
\looseness -1 The algorithm steps for \GPO largely remains the same; however, in the stochastic case, the \mdpSolver solves $\max_{\pi} m^{E}_{\pi_{t-1}}(\pi)$ using standard \RL techniques. Please see \cref{sec:appx_algorithm} for a detailed algorithm for stochastic GMDPs. While our algorithm can optimize arbitrary global rewards, in the next section we present approximation guarantees for monotone submodular, supermodular and BP rewards. 

%% file: sections/7_guarantees.tex
Due to the non-additivity of the reward function, performance guarantees cannot be based on the classic notion of value function \cite{puterman2014markov}. Nonetheless, along the lines of~\cite{tarbouriech2019active, mutti2023convex}, we define the following notion of optimality.
\begin{restatable}[Non-additive Suboptimality Gap]{definition}{suboptimalityGap}
\label{def:suboptimality_gap} 
Consider a policy $\pi \in \Pi$ interacting with a Global MDP $\cmp_F$. We define the non-additive suboptimality gap of $\pi$ as: 
\begin{equation*}
    \Rmath(\pi) \coloneqq \J^\star - \EV_{\tau \sim p_\pi} \big[ F(\tau) \big] 
\end{equation*}
where $\J^\star \coloneqq  \max_{\pi \in \Pinm} \;\;  \EV_{\tau \sim p_\pi} \big[ F(\tau) \big]$.
\end{restatable}
Since it is well-understood that in general $\PimNs$ is not sufficient to minimize the Non-additive Suboptimality Gap~\cite{prajapat2023submodular, mutti2022importance}, in the following we give guarantees \wrt the reference class $\Pinm$.

\subsection{How good is a modular approximation?}
The algorithms proposed in Section \ref{sec:method} approximately solve the GRL problem (\ref{eq:global_reinforcement_learning_traj} and \ref{eq:global_reinforcement_learning}) by optimizing a sequence of modular lower bounds of the global reward $F$.
In this section we aim to derive first-iterate approximation guarantees~\cite{iyer2013fast} to answer the question:
\begin{center}
    \emph{How well can a Global MDP be approximated by a single MDP with local rewards? }
\end{center}
Intuitively, depends on how much the global reward is \emph{non-additive}. For the case of submodular, supermodular, and BP rewards, this can be captured via the notions of \emph{submodular} and \emph{supermodular curvature} \citep{CONFORTI1984251}.
\begin{restatable}[Submodular and Supermodular Curvature]{definition}{submodularCurvature}
\label{def:submodular_curvature} 
Given a monotone set-function $F: 2^V \to \Reals$ we define the submodular and supermodular curvatures respectively as: 
\begin{align*}
    \!\!k_F \!:= 1 - \min_{v \in V} \frac{F(v \mid V \backslash \{v\})}{F(v)} \in [0,1] \tag{$F$ submodular}\\
     \!\!k^F\! := 1 - \min_{v \in V} \frac{F(v)}{F(v \mid V \backslash \{v\})} \in [0,1] \tag{$F$ supermodular} 
\end{align*}
\end{restatable}
For a BP function $F=Q+G$ we can derive curvature-based guarantees by combining the curvatures of $Q$ and $G$.
Moreover, we say that a submodular (supermodular) function $F$ has bounded curvature if $k_F (k^F) < 1$. Otherwise, we say that $F$ is fully-curved. We can now state the approximation guarantees achieved by \GPO in a general stochastic GMDP \wrt the Non-additive Suboptimality Gap (Definition \ref{def:suboptimality_gap}). 
\begin{tcolorbox}[colframe=white!, top=2pt,left=2pt,right=2pt,bottom=2pt] 
\begin{restatable}[Approximation Guarantees \GPO]{theorem}{guaranteeStochastic}
\label{theorem:guarantees_stochastic} 
\looseness -1 Let $\J^\star \coloneqq  \max\nolimits_{\pi \in \Pinm} \;\;  \EV_{\tau \sim p_\pi} \big[ F(\tau) \big]$ and $\pi_1$ the policy resulting from one iteration of \GPO on a GMDP $\gmdp$. Then \GPO guarantees that for 
\begin{enumerate}[label=\textit{\roman*}),noitemsep]
\item Monotone submodular reward function, $F$
    \begin{align*}
    \Rmath(\pi_1) &\leq k_F \J^\star,
    \end{align*}
\item Monotone supermodular reward function, $F$   
    \begin{align*}
    \Rmath(\pi_1) &\leq \frac{2k^F-(k^F)^2}{1-k^F} \J^\star,
    \end{align*}
\item BP reward function, $F=Q+G$  
\begin{align*}
\!\!\!\! \Rmath(\pi_1) &\leq \alpha \J^\star,~\alpha  =
    \begin{cases}
    \frac{2k^G-(k^G)^2}{1-k^G}  & \!\text{if $k_F\leq k^G$} \\
    \frac{1-(1-k_Q)(1-k^G)}{1-k^G} & \!\text{otherwise} 
    \end{cases} 
\end{align*}
\end{enumerate}
\end{restatable}
\end{tcolorbox}
\looseness -1 Interestingly, the notion of Non-additive Suboptimality Gap (\cref{def:suboptimality_gap}) can be tailored to the trajectory-optimization version of GRL  \eqref{eq:global_reinforcement_learning_traj}, and similar guarantees for \GTO can be derived for deterministic GMDPs. The analysis for this case and the proof of Theorem \ref{theorem:guarantees_stochastic} are in \cref{sec:det_case_appendix,sec:proofs}.

\looseness -1 \mypar{Discussion} \cref{theorem:guarantees_stochastic} relates the suboptimality gap achieved by Markovian policies with the degree of non-additivity of the reward function as captured by curvature.\footnote{Computable in linear time \wrt $|\Sspace \times \T|$.} Notably, the submodular functions enjoy the best approximation ratio, while high supermodular curvature seems to substantially affect the solution quality. Moreover, notice that since \GPO optimizes a sequence of lower bounds, the quality of the returned solutions can be better than the ones stated in Theorem~\ref{theorem:guarantees_stochastic}, as shown in the experiments (Section \ref{sec:experiments}). Ultimately, notice that these results extend classic submodular function maximization results to the case of constrained optimization under dynamics constraints, as presented in Section \ref{sec:grl_as_co}.

\subsection{Hardness of Global RL}
\label{sec:hardness}
\input{sections/hardness}

%% file: sections/hardness.tex
\begin{figure*}[t] \setlength{\abovecaptionskip}{4pt}
\centering
\subcaptionbox[Short Subcaption]{
     Experiment design
    \label{subfig:d_experimental_design_det}}
[
    0.33\textwidth 
]
{%
    \includegraphics[width=0.3\textwidth]{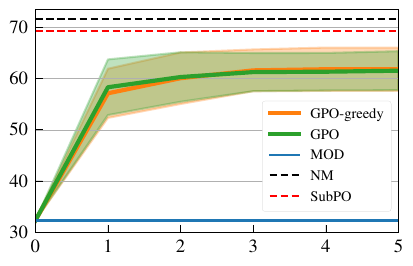}%
}%
\hfill 
\subcaptionbox[Short Subcaption]{%
    Diverse synergies
    \label{subfig:diverse_synergies_stoch}%
}
[%
    0.33\textwidth 
]%
{%
    \includegraphics[width=0.3\textwidth]{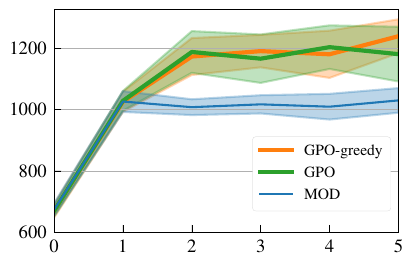}%
}%
\hfill
\subcaptionbox[Short Subcaption]{%
    Safe state coverage
    \label{subfig:indicator_safe_cover_1_det}%
}
[%
    0.33\textwidth 
]%
{%
    \includegraphics[width=0.3\textwidth]{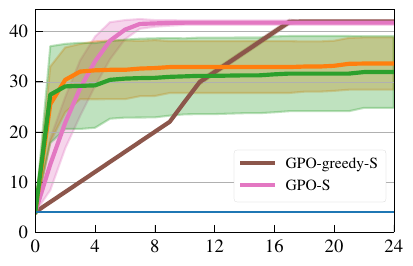}%
}%
\caption[Short Caption]{We compare \GTO and \GPO with the optimal policy for the modularized objective $F_m$ (MOD). We observe that MOD performs sub-optimally as its objective cannot capture interactions between states. The alternative versions of the algorithm tested are presented in Section \ref{sec:alternative_LB}.   (Y-axis: $\J(\pi)$, X-axis: iterations)}
\label{fig:label} \vspace{-0em}
\end{figure*}

To conclude this section, we present a hardness result for the trajectory-optimization GRL Problem \eqref{eq:global_reinforcement_learning_traj} in deterministic environments. The following result extends known results for BP-function maximization~\cite{bai2018greed} to the case of dynamics-constrained optimization. 

\begin{tcolorbox}[colframe=white!, top=2pt,left=2pt,right=2pt,bottom=2pt]
\begin{restatable}[Hardness of GRL, trajectory-optimization \eqref{eq:global_reinforcement_learning_traj}]{theorem}{hardness}
\label{theorem: hardness}
\looseness -1 For all $0 \leq \beta \leq 1$, there exists an instance of a BP global reward $F = Q + G$ with $k^G = \beta$ such that no poly-time algorithm can achieve an approximation factor better than $1-k^G + \epsilon \ \forall \epsilon > 0$ \wrt the Non-additive Suboptimality Gap in deterministic GMDPs, unless P = NP.
\end{restatable}
\end{tcolorbox}

The proof of \cref{theorem: hardness} can be found in \cref{sec:proof_hardness}. Crucially, theorem \ref{theorem: hardness} shows that the supermodular curvature plays a fundamental role in determining a lower bound on the approximability of the problem. 



%% file: sections/8_experiments.tex
\looseness -1 We present an experimental analysis of \GTO and \GPO on three tasks: \textit{i)} D-optimal experimental design, \textit{ii)} diverse and synergical trajectory selection, and \textit{iii)} safe state coverage. These environments have deterministic (\textit{i,iii}) and stochastic (\textit{ii}) transitions, consider global rewards with bounded curvature (\textit{i}) as well as fully-curved (\textit{ii,iii}), and cover submodular (\textit{i}), BP (\textit{ii}), and arbitrary global rewards (\textit{iii}).

\looseness -1 We compare the performances of the policy $\pi$ obtained via \GTO  (\cref{alg:global_traj_optim}) and \GPO (\cref{alg:GPO_algorithm}) with the performance of an optimal policy \wrt the additive objective $\J_m(\pi) = \EV_{\tau \sim p_\pi} \big[\sum_{v \in \tau} F(v) \big]$, which disregards interactions between states within the same trajectory, as it is the case in classic RL. Moreover, wherever possible\footnote{For computational reasons, we can solve it in deterministic environments with monotonic non-decreasing global rewards.}, we compare the performance of $\pi$ with the optimal non-Markovian policy \wrt $F$, and with SubPO~\cite{prajapat2023submodular}, a policy gradient method for optimizing under submodular rewards. 
We do not explicitly compare with other RL algorithms, since, SubPO is a representative baseline for Reinforce-type algorithms with variance reduction techniques catering to non-additive returns. 
Notably, for efficient performance, we run SubPO with policy parameterized with a neural network and thus it loses its approximation guarantees in contrast to \GPO. Moreover, we consider two variants of our algorithm $\GPO$-S and $\GPO$-greedy which build lower bounds based state-dependent and greedy permutation respectively (see \cref{sec:alternative_LB} for details).
We run all experiments over a grid environment with 400 states, discrete actions $\{$left, right, up, down, stay$\}$, and plot $95\%$ confidence intervals over $20$ runs.

\looseness -1 \mypar{Bayesian D-Optimal experimental design} In this task, the agent aims to optimally select a trajectory over sampling points to estimate an unknown function $f$ represented via Gaussian Processes~\cite{rasmussen2006gaussian}, a problem known as optimal experimental design~\cite{mutny2023active}. We aim to maximize the mutual information $I(y_\tau; f) = H(y_\tau) - H(y_\tau | f)$ between the observations $y_\tau$, and the unknown function $f$ evaluated at locations in $\tau$, which is a submodular global reward. 
In \cref{subfig:d_experimental_design_det}, we observe that \GTO performs significantly better than the optimal policy for the modularized objective, competitively with SubPO, and nearly optimally \wrt to the optimal non-Markovian policy.

\looseness -1 \mypar{Diverse synergies} In this task, we maximize the BP rewards $F(\tau) = |\bigcup_{s\in \tau} D^s| + \sum_{i=1}^K |\tau \cap S_i|^\beta$ with $S_i \subseteq \V \coloneqq \Sspace \times T$, in a stochastic environment, where with $0.1$ probability the agent transitions to a neighbouring state uniformly at random. This objective induces policies maximizing coverage over state space while seeking specific complementarity between states within the trajectory resembling, for instance, synergetic effects between atoms in a molecular design process. As hinted by \cref{theorem: hardness,theorem:guarantees_stochastic}, maximizing the fully-curved supermodular component of $F$ may perform arbitrarily poorly. Nonetheless, in \cref{subfig:diverse_synergies_stoch} we observe that \GPO performs significantly better than the modular optimal policy \wrt $F_m$. Moreover, in \cref{sec:experiments_details}, we show that in the deterministic setting, where we can compute the optimal non-Markovian policy, \GTO performs optimally.

\looseness -1 \mypar{Safe state coverage} In \cref{subfig:indicator_safe_cover_1_det}, we consider a non-monotone BP function encoding the task of safe state space coverage defined as $|\bigcup_{s\in \tau} D^s| + C \cdot \indicator \{\tau \cap S_{u} = 0\}$ with $D^s$ being a disk covering neighbouring states \citep{near_optimal_safe_cov}.
An optimal policy \wrt this objective is highly explorative while avoiding unsafe areas. Since this objective is invariant \wrt the time-dimension, we can leverage the lower bounds introduced in \cref{sec:alternative_LB}. 
We observe that the policy returned by \GTO performs significantly better than the optimal policy \wrt to the modularized objective $F_m$. This is due to the fact that the task embeds a notion of diversity that cannot be captured by local reward functions.

\subsection{Experimental Insights and Observations}
\label{sec:insights}
In the following, we aim to present several claims emerging from the experiments presented in Section \ref{sec:experiments} and in Appendix \ref{sec:experiments_details}.

\mypar{Practical versus theoretical performances}
\looseness -1 In the case of bounded-curvature reward functions, the theory presented in Section \ref{sec:guarantees} gives first-iteration performance guarantees for \GTO and \GPO. Nonetheless, in practice these algorithms can significantly improve the solution after the first iteration thus potentially reach significantly better performances than the ones captured via such curvature-based guarantees. The experimental results show that this fact is particularly true for submodular global reward functions where the objective value increases significantly over multiple iterations.

\mypar{Beyond bounded-curvature functions}
While the curvature assumption is needed to build theoretical guarantees, the proposed algorithms seem to show outstanding performances even with fully-curved reward functions. 

\mypar{Non-monotonic improvements in \GPO}
As pointed out in Section \ref{sec:method}, in the stochastic setting \GPO does not guarantee a monotonic performance improvement, but we still have a curvature-based performance guarantee as shown in Section \ref{sec:guarantees}. In the experiments above, we show that in practice \GPO shows promising performances even in settings with stochastic dynamics. Clearly, concentrating the estimates in \GPO requires sampling multiple trajectories thus increasing the computational complexity of the algorithm.

\mypar{Greedy submodular lower bounds} Observing the performances of greedy lower bounds (indicated with $-greedy$ in Fig. \ref{fig:label} and in Appendix \ref{sec:experiments_details}), we can notice that they outperform the normal counterpart in certain instances. In particular, a further analysis reported in Fig. \ref{fig:cover_1_traj} within Appendix \ref{sec:experiments_details}, shows that this is the case when the horizon is neither too large nor too small, but just enough to properly maximize a certain submodular reward via a specific trajectory (or policy), as shown in Figure \ref{fig:cover_1_traj} (right). Meanwhile, when the horizon is arbitrarily small, the two lower bounds seem to perform equally well. This is due to the fact that in this latter case random actions will also perform well as there is less chance of overlapping states and therefore better planning strategies cannot show performance improvements.

%% file: sections/9_related_works.tex
\looseness -1 \mypar{Convex \RL}
Convex RL (CRL)~\cite{hazan2019maxent,zahavy2021reward} is a recent framework that extends RL to non-additive rewards by optimizing convex functionals of state(-action) distributions. It can capture a wide range of applications including exploration~\cite{hazan2019maxent}, experimental design~\cite{Mutny2023}, imitation learning~\cite{state_marginal}, and risk-averse RL~\cite{garcia2015comprehensive, mutti2022challenging}. Interestingly, also common CRL algorithms reduce the problem to solving a sequence of MDPs. Moreover, notice that several CRL objectives, \eg entropy maximization and experimental design, can be cast as global rewards as shown in \cref{table:GRLapplications_all}. Furthermore, while (mixtures of) Markovian stationary policies are sufficient to optimize CRL objectives, Global RL leverages Markovian non-stationary policies, a fact that manifests the need of more general policy classes to optimize non-additive rewards in finite-samples settings, as shown by~\citet{mutti2022importance}. Further, similar to some Convex RL works \cite{tarbouriech2019active, zahavy2021reward}, we believe it is possible, and an interesting direction of future work, to extend \GPO to the case of unknown dynamics, via an optimistic estimate of $P$.

\looseness -1 \mypar{RL with Non-Markovian Rewards}
The concepts of global rewards and non-additive rewards presented within this work have a strong link with the notion of non-Markovian rewards, namely history-dependent reward functions. Non-Markovian rewards in RL has been studied from the lenses of logic and formal language theory~\cite{gaon2020reinforcement, de2013linear, brafman2018ltlf}. Currently, the analysis within these works seems orthogonal to ours; however, exploring connections between these two approaches is an interesting research direction.

\looseness -1 \mypar{Submodularity in Decision Making}
Submodular functions have been used to model problems in active machine learning~\cite{krause2014submodular,bilmes2022submodularity}, bandits~\citep{yue2011linear, Chen2017InteractiveSB, gabillon2013adaptive}, planning~\citep{chekuri_rg,wang2020planning} and RL~\cite{prajapat2023submodular}. The most related to us is from \citet{prajapat2023submodular} which introduces Submodular MDPs and proposes a policy gradient (PG) based method for it. Beyond the larger generality of the function classes we consider, the main difference is that we transport an algorithmic scheme from submodular optimization tailored for BP function maximization to MDPs, rather than using a function-agnostic optimization scheme like PG.

\looseness -1 \mypar{BP function maximization}
Historically, submodular and supermodular function maximization have been treated with significantly different approaches \citep{krause2014submodular, iwata2008submodular}. Recently, novel optimization schemes based on discrete subgradients~\cite{iyer2013fast, iyer2015polyhedral} made it possible to treat submodular and supermodular function maximization in a unified manner, directly optimize suBmodular-suPermodular (BP) functions \citep{bai2018greed}, as well as optimize arbitrary set functions with a known decomposition into a submodular and a supermodular component. 
In \cref{sec:guarantees} we extend hardness results for BP function maximization to GRL by shedding light on the effect of the transition dynamics underlying the Markov process. This removes the possibility of \emph{teleporting} from any element of the set to any other element, which we capture in \cref{sec:grl_as_co} via the notion of dynamics constraint (see \cref{sec:dynamics_constraint}).

%% file: sections/10_conclusions.tex
\looseness -1 We introduce a novel problem formulation denoted {\em Global} RL (GRL), for sequential decision-making under non-additive objectives. Then, leveraging tools from submodular optimization, we provide a novel algorithmic scheme that converts the GRL problem to a sequence of standard RL planning problems. We identify structural properties that render the problem efficiently and approximately solvable for a wide variety of applications according to curvature-based approximation guarantees. Moreover, we derive a hardness result for a representative class of GRL and extensively showcase the empirical performances of the proposed algorithms to solve RL problems that cannot be captured via classic additive (or local) rewards.

%% file: sections/dynamics_constraint.tex
Towards interpreting GRL as a CO problem, we introduce a type of constraint set that we call \emph{dynamics constraint}, and a trajectory-optimization version of GRL.
But first, we need to introduce the following auxiliary notion.
\begin{restatable}[Time-extended CMP]{definition}{timeExtendedCMP}
    Given a CMP $\mathcal{M} \coloneqq \langle\Sspace, \Aspace, P, H, \mu \rangle$, we call time-extended CMP, the new CMP $\mathcal{M}_l$ defined by $\mathcal{M}_l \coloneqq \langle V \coloneqq \Sspace \times [l] , \Aspace, P_l, H, \mu\rangle$, which is a $l$-layered CMP where $P_l$ is defined as:
     \[
    P_l((s',t') \mid (s,t), a)  =
    \begin{cases}
    P(s' \mid s, a) & \text{if $t'=t+1$} \\
    0 & \text{otherwise} 
    \end{cases}
    \]
\end{restatable}
Now we can define the dynamics constraint $\mathcal{C}_{\mathcal{M}}$ as follows.
\begin{restatable}[Dynamics constraint $\mathcal{C}_{\mathcal{M}}$]{definition}{CMPConstraint}
\label{def: cmp_constraint}
    Given a CMP $\mathcal{M} \coloneqq \langle \Sspace, \Aspace, P, H, \mu \rangle$, we consider the time-extended CMP $\mathcal{M}_H \coloneqq \langle V \coloneqq \Sspace \times \T , \Aspace, P_H, H, \mu \rangle$ and define $\mathcal{C}_{\mathcal{M}}$ as:
    \begin{equation*}
        \mathcal{C}_{\mathcal{M}} \coloneqq \{ \tau \subseteq V \mid \tau \text{ is a path in } \mathcal{G}_P \}
    \end{equation*}
    where $\mathcal{G}_P = (V, E)$ is the graph induced by $P_H$ where $e = ((s,t),(s',t')) \in E \iff \exists a \in \Aspace \text{ s.t. } P_H((s',t') \mid (s,t), a))>0$.
\end{restatable}
Notice we can interpret a dynamics constraint $\mathcal{C}_{\mathcal{M}}$ as the set of admissible trajectories of a CMP $\mathcal{M}$. 

%% file: sections/proofs.tex
\subsection{Proofs for Section \ref{sec:general_utilities}}
\equivalence*
\begin{proof}
For the sake of clarity, without loss of generality, we consider an empirical distribution defined over $\Sspace \times \T$. Notice that this is indeed general since in the case the original instance does not take into account the time dimension, \eg optimizes over distributions $d \in \Delta(\Sspace)$, then it is sufficient to define a functional $\F'$ marginalizing the input distribution, \eg $d' \in \Delta(\Sspace \times \T)$ over the time dimension. Moreover, by considering trajectories composed of state-action pairs rather than only states, it is trivial to extend the following proof to the case of distributions over $\Sspace \times \Aspace \times \T$.
Given an instance of the Single Trial General Utilities RL, namely:
\begin{equation}
\label{single_trial_general_utilities_rl_apx}
    \max_{\pi \in \Pi} \;\; \Big( \EV_{d \sim p^\pi} \big[ \F(d) \big] \Big)
\end{equation}
we  show how to build an equivalent instance of the GRL problem, namely:
\begin{equation}
\label{eq:global_reinforcement_learning_apx}
    \max_{\pi \in \Pi} \;\;  \EV_{\tau \sim p_\pi} \big[ F(\tau) \big] 
\end{equation}
such that the two problems are equivalent. In particular, we prove a stronger result than equivalence on the set of maximizers. We show that the objective functions are equal, formally:
$$\EV_{d \sim p^\pi} \big[ \F(d) \big] = \EV_{\tau \sim p_\pi} \big[ F(\tau) \big]  \quad \forall \pi \in \Pi$$
We build the instance of the GRL problem as follows. First, we realize that every empirical distribution $d \sim p^\pi$ must be induced by a trajectory $\tau$ following policy $\pi$ due to the definition of $p^\pi$. Hence, for the sake of clarity, we express a distribution of this type as $d_\tau$. As a consequence, for every empirical distribution $d_\tau$ we can define the trajectory $\tau$ inducing that distribution. This is trivial since $d_\tau$ is fundamentally an indicator function over $\Sspace \times \Aspace \times \T$ such that $d_\tau(s,t) = \mathbf{1}_{\tau(t) = (s)}$, where $\tau(t)$ represent the state visited within the trajectory $\tau$ at time step $t$. Since there exists a bijection between the space of trajectories $\tau$ and the space of empirical distributions $d_\tau$, we can express the probability of sampling a certain trajectory $\tau$, namely $p_\tau$ as follows:
\begin{equation}
\label{eq:p_equality}
    p_\pi(\tau) \coloneqq p^\pi(d_\tau)
\end{equation}
For the sake of clarity, we prove Equation \eqref{eq:p_equality} explicitly by leveraging the definitions of $p_\pi(\tau)$ and $p^\pi(d_\tau)$. In particular, given a policy $\pi$, we want to prove that for any trajectory $\tau$ it holds that $p_{\pi}(\tau) = p^{\pi}(d_{\tau})$. Consider a trajectory $\tau = \{((s_0,0),a_0),((s_1,1),a_1),\dots, ((s_t,t),a_t),\dots, ((s_{H-1}, H-1),a_{H-1})\}$ and the empirical distribution $d_{\tau} \in \Delta(\V)=\Delta(\Sspace \times \T)$. Consider the function $d_{\tau}^t = (s,t) \in \V : d_{\tau}>0$, which returns the state-time pair $(s,t)$ visited by trajectory $\tau$ in time-step $t$. Now we can express $p_\pi(\tau)$ and $p^\pi(d_\tau)$ as follows:
\begin{align}
    p_\pi(\tau) &= \mu((s_0,0))\prod_{t=0}^{H-1} P((s_{t+1}, t+1) | (s_t, t), a_t) \pi(a_t | (s_t,t))\\
    p^\pi(d_{\tau}) &= \mu(d^0_{\tau})\prod_{t=0}^{H-1} P(d^{t+1}_{\tau}| d^{t}_{\tau}, a_t) \pi(a_t | d^{t}_{\tau})
\end{align}
Crucially, it is immediate to notice that, due to the definition of $d_{\tau}^t \forall t \in \T$ we have that:
\begin{align*}
    \mu((s_0,0)) &= \mu(d^0_{\tau})\\
    P((s_{t+1}, t+1) | (s_t, t), a_t) &= P(d^{t+1}_{\tau}| d^{t}_{\tau}, a_t)\\
    \pi(a_t | (s_t,t)) &= \pi(a_t | d^{t}_{\tau})
\end{align*}
Ultimately, we define the GRL global reward $F$ as:
\begin{equation*}
    F(\tau) = \F(d_\tau)
\end{equation*}
Hence, by construction of the GRL instance, we have that $\forall \pi \in \Pi$:
\begin{align*}
    \EV_{d \sim p^\pi} \big[ \F(d) \big] &= \sum_{d} \F(d) p^\pi(d)\\
    &= \sum_{d_\tau} \F(d_\tau) p^\pi(d_\tau)\\
    &= \sum_{d_\tau} \F(d_\tau) p_\pi(\tau)\\
    &= \sum_{\tau} F(\tau) p_\pi(\tau)\\
    &= \EV_{\tau \sim p_\pi} \big[ F(\tau) \big]
\end{align*}
Intuitively, it is likely possible to show also the other direction of the inclusion, \ie GRL $\subseteq$ Single Trial General Utilities RL, and therefore that the two problem classes are equivalent. Although interesting, this is not necessary in order to prove the theorem stated.
\end{proof}

\subsection{Proofs for Section \ref{sec:guarantees}}
\guaranteeStochastic*
\begin{proof}
    The proof is composed of three sequential parts proving each one of the three cases independently. We first show that $\mathcal{J}(\pi_1) \geq (1-k_F)\mathcal{J}(\pi^*)$ where $\J(\pi) = \EV_{\tau \sim \pi}\big[F(\tau)\big]$. Once this is proved, the theorem statement can be trivially derived by using the definition of $\V(\pi)$.
    \begin{align*}
        \mathcal{J}(\pi_1) &\labelrel\geq{step:stoch_sub_1} m^{\sigma, E}_{\pi_0}(\pi_1) && \tag{lower bound}\\
        &\labelrel\geq{step:stoch_sub_2} m^{\sigma, E}_{\pi_0}(\pi^*) && \tag{$\pi_1$ maximizer of $m^{\sigma, E}_{\pi_0}$}\\
        &\labelrel={step:stoch_sub_3} \EV_{\tau^* \sim \pi^*} \bigg[ \EV_{\tau_0 \sim \pi_0} \bigg[ m^\sigma_{\tau_0}(\tau^*)\bigg] \bigg] && \tag{def. $m^{\sigma, E}_{\pi_0}(\pi^*)$}\\
        &= \EV_{\tau^* \sim \pi^*} \bigg[ \EV_{\tau_0 \sim \pi_0} \bigg[\sum_{x \in \tau^*} m^\sigma_{\tau_0}(I_{\tau_0}(x))\bigg] \bigg]\\
        &= \EV_{\tau^* \sim \pi^*} \bigg[ \EV_{\tau_0 \sim \pi_0} \bigg[\sum_{x \in \tau^*} F\big(S^{\sigma(\tau_0)}_{I_{\tau_0}(x)}\big) - F\big(S^{\sigma(\tau_0)}_{I_{\tau_0}(x)-1}\big) \bigg] \bigg]\\
        &=\EV_{\tau^* \sim \pi^*} \bigg[ \EV_{\tau_0 \sim \pi_0} \bigg[ \sum_{x \in \tau^*} F(\{\sigma(\tau_0)[1], \ldots, \sigma(\tau_0)[I_{\tau_0}(x)]\}) - F(\{\sigma(\tau_0)[1], \ldots, \sigma(\tau_0)[I_{\tau_0}(x)-1]\}) \bigg] \bigg]\\
        &=\EV_{\tau^* \sim \pi^*} \bigg[ \EV_{\tau_0 \sim \pi_0} \bigg[ \sum_{x \in \tau^*} F(\sigma(\tau_0)[I_{\tau_0}(x)] \mid \{\sigma(\tau_0)[1], \ldots, \sigma(\tau_0)[I_{\tau_0}(x)-1]\}) \bigg] \bigg]\\
        &\geq (1-k_F)\EV_{\tau^* \sim \pi^*} \bigg[ \EV_{\tau_0 \sim \pi_0} \bigg[ \sum_{x \in \tau^*} F(\sigma(\tau_0)[I_{\tau_0}(x)]) \bigg] \bigg]\\
        &\labelrel={step:stoch_sub_4} (1-k_F)\EV_{\tau^* \sim \pi^*} \bigg[ \EV_{\tau_0 \sim \pi_0} \bigg[ \sum_{x \in \tau^*} F(x) \bigg] \bigg]\\
        &\geq (1-k_F)\EV_{\tau^* \sim \pi^*} \bigg[ \EV_{\tau_0 \sim \pi_0} \bigg[ F(\tau^*) \bigg] \bigg]\\
        &\labelrel\geq{step:stoch_sub_5} (1-k_F)\EV_{\tau^* \sim \pi^*} \bigg[ F(\tau^*) \bigg]\\
        &= (1-k_F)\mathcal{J}(\pi^*)
    \end{align*}
We define $I_{\tau}(x)$ as the function returning the index of element $x$ in trajectory $\tau$. This proof is heavily based on the proof of theorem \ref{theorem:guarantee_submodular_deterministic}. Analogously, in step \eqref{step:stoch_sub_1} we use the fact that $m^{\sigma, E}_{\pi_0}(\pi_1)$ is by construction a lower bound of $\mathcal{J}(\pi_1)$, in step \eqref{step:stoch_sub_2} we notice that $\pi_1$ is a maximizer of $m^{\sigma, E}_{\pi_0}(\cdot)$ due to the local optimization step in \GPO, while in step \eqref{step:stoch_sub_3} we leverage the definition of $m^{\sigma, E}_{\pi_0}(\pi^*)$. In step \eqref{step:stoch_sub_4} we leverage the notion of submodular curvature (Definition \ref{def:submodular_curvature}), and in step \eqref{step:stoch_sub_5} we exploit submodularity of $F$
 (Definition \ref{def:submodularity}).

    We now prove the second statement. This proof is heavily based on the proof of theorem \ref{theorem:guarantee_supermodular_deterministic}. We first show that $\mathcal{J}(\pi_1) \geq \frac{(k^F)^2 - 3k^F + 1}{1-k^F}\mathcal{J}(\pi^*)$ where $\J(\pi) = \EV_{\tau \sim \pi}\big[F(\tau)\big]$. Once this is proved, the theorem statement can be trivially derived by using the definition of $\V(\pi)$.
    \begin{align}
        \mathcal{J}(\pi_1) &\labelrel\geq{step:stoch_sup_1} m^{E}_{\pi_0}(\pi_1)\nonumber\\
        &\labelrel\geq{step:stoch_sup_2} m^{E}_{\pi_0}(\pi^*)\nonumber\\
        &\labelrel={step:stoch_sup_3} \EV_{\tau^* \sim \pi^*} \bigg[ \EV_{\tau_0 \sim \pi_0} \bigg[ m_{\tau_0}(\tau^*)\bigg] \bigg] \nonumber\\
        &\labelrel\geq{step:stoch_sup_4} \EV_{\tau^* \sim \pi^*} \bigg[ \EV_{\tau_0 \sim \pi_0} \bigg[ (1-k^F)F(\tau^*) - \frac{k^F}{1-k^F}\sum_{j \in \tau_0 \backslash \tau^*}F(j) \bigg] \bigg] \nonumber\\
        &\labelrel\geq{step:stoch_sup_5} (1-k^F)\EV_{\tau^* \sim \pi^*}\bigg[ \EV_{\tau_0 \sim \pi_0} \bigg[ F(\tau^*)\bigg]\bigg] - \frac{k^F}{1-k^F} \EV_{\tau^* \sim \pi^*}\bigg[ \EV_{\tau_0 \sim \pi_0} \bigg[ \sum_{j \in \tau_0 \backslash \tau^*} F(j) \bigg]\bigg] \label{eq: middle_step_apx_stoch_sup} 
    \end{align}
    In step \eqref{step:stoch_sup_1} we use the fact that $m^{E}_{\pi_0}(\pi_1)$ is by construction a lower bound of  $\mathcal{J}(\pi_1)$, in step \eqref{step:stoch_sup_2} we notice that $\pi_1$ is a maximizer of $m^{E}_{\pi_0}(\cdot)$ due to the local optimization step in \GPO, while in step \eqref{step:stoch_sup_3} we leverage the definition of $m^{E}_{\pi_0}(\pi^*)$. Meanwhile, in step \eqref{step:stoch_sub_4}, we trivially follow steps  \eqref{step:det_sup_1} to  \eqref{step:det_sup_1} of the proof of theorem \ref{theorem:guarantee_supermodular_deterministic}. Now, we lower bound the second term of equation \ref{eq: middle_step_apx_stoch_sup} (without the preceding constant component) as follows:
    \begin{align}
        - \EV_{\tau^* \sim \pi^*}\bigg[ \EV_{\tau_0 \sim \pi_0} \bigg[ \sum_{j \in \tau_0 \backslash \tau^*} F(j) \bigg]\bigg] &\geq - \EV_{\tau^* \sim \pi^*}\bigg[ \EV_{\tau_0 \sim \pi_0} \bigg[ F(\tau_0) \bigg]\bigg] \nonumber\\
        &= -\EV_{\tau_0 \sim \pi_0} \bigg[ F(\tau_0) \bigg] \nonumber\\
        &= -\mathcal{J}(\pi_0) \nonumber\\
        &\geq - J(\pi^*) \label{eq: side_step_apx_stoch_sup}
    \end{align}
    By plugging the result in eq. \ref{eq: side_step_apx_stoch_sup} into equation \ref{eq: middle_step_apx_stoch_sup}, we obtain:
    \begin{align*}
        \mathcal{J}(\pi_1) &\geq (1-k^F)\mathcal{J}(\pi^*) - \frac{k^F}{1-k^F}\mathcal{J}(\pi^*)\\
        &= \frac{(k^F)^2 - 3k^F + 1}{1-k^F} \mathcal{J}(\pi^*)
    \end{align*}
    where we have employed the definition of $\mathcal{J}(\pi^*)$.
    Notice that a critical aspect of this proof is that we cannot compare $F(\tau^*)$ and $F(\tau)$ as we did for the proofs in the deterministic setting, since in the stochastic setting the notion of optimality is defined over the policy space and therefore we can only compare 
    \begin{equation*}
        \mathcal{J}(\pi^*) = \EV_{\tau^* \sim \pi^*} \bigg[ F(\tau^*)\bigg] \geq \EV_{\tau \sim \pi} \bigg[ F(\tau) \bigg] = \mathcal{J}(\pi) \quad \forall \pi \in \Pi
    \end{equation*}
    We now prove the third statement. First, notice that:
    \begin{align}
        \J(\pi) &= \EV_{\tau \sim \pi} \bigg[ F(\tau)\bigg] \nonumber\\
        &= \EV_{\tau \sim \pi} \bigg[ Q(\tau) + G(\tau)\bigg] \nonumber\\
        &= \EV_{\tau \sim \pi} \bigg[ Q(\tau)\bigg] + \EV_{\tau \sim \pi} \bigg[ G(\tau)\bigg] \nonumber\\
        &= {\J}_Q(\pi) + {\J}_G(\pi) \label{eq:decomposition_J}
    \end{align}
    where we define:
    \begin{equation*}
        {\J}_Q \coloneqq \EV_{\tau \sim \pi} \bigg[ Q(\tau)\bigg] \quad\quad {\J}_G \coloneqq \EV_{\tau \sim \pi} \bigg[ G(\tau)\bigg]
    \end{equation*}
    and:
    \begin{equation*}
        \pi^* \in \argmax_\pi \J(\pi) \quad \pi^*_Q \in \argmax_\pi {\J}_Q(\pi) \quad \pi^*_G \in \argmax_\pi {\J}_G(\pi)
    \end{equation*}
     Given these quantities, we can proceed with the core of the proof. We first show that $\mathcal{J}(\pi_1) \geq \alpha \cdot \mathcal{J}(\pi^*)$ for the value of $\alpha$ stated within the theorem. From this, we deduce that $\mathcal{J}(\hat{\pi}) \geq \alpha \cdot \mathcal{J}(\pi^*)$ since further policy updates are performed only if they don't worsen the policy performance according to $\mathcal{J}$.
    \begin{align}
        \J(\pi_1) &\labelrel={step:stoch_bp_1} {\J}_Q(\pi_1) + {\J}_G(\pi_1) \nonumber\\
        &\labelrel\geq{step:stoch_bp_2} m_{\pi_0}^{\sigma, E}(\pi_1) + m_{\pi_0}^{E}(\pi_1) \nonumber\\
        &\labelrel\geq{step:stoch_bp_3} m_{\pi_0}^{\sigma, E}(\pi^*) + m_{\pi_0}^{E}(\pi^*) \nonumber\\
        &\labelrel={step:stoch_bp_4} \EV_{\tau^* \sim \pi^*} \bigg[ \EV_{\tau_0 \sim \pi_0} \bigg[ m^\sigma_{\tau_0}(\tau^*)\bigg] \bigg] + \EV_{\tau^* \sim \pi^*} \bigg[ \EV_{\tau_0 \sim \pi_0} \bigg[ m_{\tau_0}(\tau^*)\bigg] \bigg] \nonumber\\
        &= \EV_{\tau^* \sim \pi^*} \bigg[ \EV_{\tau_0 \sim \pi_0} \bigg[ m^\sigma_{\tau_0}(\tau^*) + m_{\tau_0}(\tau^*)\bigg] \bigg] \nonumber\\
        &\labelrel\geq{step:stoch_bp_5} \EV_{\tau^* \sim \pi^*} \bigg[ \EV_{\tau_0 \sim \pi_0} \bigg[ (1-k_Q)Q(\tau^*) + m_{\tau_0}(\tau^*)\bigg] \bigg] \nonumber\\
        &= (1-k_Q){\J}_Q(\pi^*) + \EV_{\tau^* \sim \pi^*} \bigg[ \EV_{\tau_0 \sim \pi_0} \bigg[ m_{\tau_0}(\tau^*)\bigg] \bigg] \label{eq: middle_step_BP_stoch}
    \end{align}
     In step \eqref{step:stoch_bp_1} we use equation \eqref{eq:decomposition_J}, in step \eqref{step:stoch_bp_2} we use the fact that $m_{\pi_0}^{\sigma, E}(\pi_1)$ and $m_{\pi_0}^{E}(\pi_1)$ are by construction respectively lower bounds of ${\J}_Q(\pi_1)$ and ${\J}_G(\pi_1)$. In step \eqref{step:stoch_bp_3}, we notice that $\pi_1$ is a maximizer of the sum of marginal lower bounds, namely $m_{\pi_0}^{\sigma, E}(\pi^*) + m_{\pi_0}^{E}(\pi^*)$ due to the local optimization step in \GPO. Meanwhile in step \eqref{step:stoch_bp_4} we leverage the definition of $ m_{\pi_0}^{\sigma, E}(\pi^*)$ and $m_{\pi_0}^{E}(\pi^*)$, and step \eqref{step:stoch_bp_5} can be trivially derived by following steps \eqref{step:det_sub_2} to \eqref{step:det_sub_4} used to prove theorem \ref{theorem:guarantee_submodular_deterministic}.
    Now we can lower bound the second term of equation \ref{eq: middle_step_BP_stoch} as:
    \begin{align}
        \EV_{\tau^* \sim \pi^*} \bigg[ \EV_{\tau_0 \sim \pi_0} \bigg[ m_{\tau_0}(\tau^*)\bigg] \bigg] &\labelrel\geq{step:stoch_bp_6} \EV_{\tau^* \sim \pi^*} \bigg[ \EV_{\tau_0 \sim \pi_0} \bigg[ (1-k^G)G(\tau^*) - \frac{k^G}{1-k^G}\sum_{j \in \tau_0 \backslash \tau^*}G(j) \bigg] \bigg] \nonumber\\
        &= (1-k^G){\J}_G(\pi^*) - \frac{k^G}{1-k^G} \EV_{\tau^* \sim \pi^*} \bigg[ \EV_{\tau_0 \sim \pi_0} \bigg[ \sum_{j \in \tau_0 \backslash \tau^*} G(j)\bigg] \bigg] \label{eq: side_step_BP_stoch}
    \end{align}
    where step \eqref{step:stoch_bp_6} can be trivially derived by following steps \eqref{step:det_sup_2} to \eqref{step:det_sup_4} used to prove theorem \ref{theorem:guarantee_supermodular_deterministic}. Now we can lower bound the second term of equation \ref{eq: side_step_BP_stoch} (without the preceding constant component) as follows:
    \begin{align}
        - \EV_{\tau^* \sim \pi^*} \bigg[ \EV_{\tau_0 \sim \pi_0} \bigg[ \sum_{j \in \tau_0 \backslash \tau^*} G(j)\bigg] \bigg] &\labelrel\geq{step:stoch_bp_7} - \EV_{\tau^* \sim \pi^*} \bigg[ \EV_{\tau_0 \sim \pi_0} \bigg[ G(\tau_0)\bigg] \bigg] \nonumber\\
        &= -{\J}_Q(\pi_0) \nonumber\\
        &\geq -{\J}_Q(\pi_G^*) \nonumber\\
        &\labelrel\geq{step:stoch_bp_8} -\J(\pi^*) \label{eq: last_step_BP_stoch}
    \end{align}
    using in step \eqref{step:stoch_bp_7} the fact that $(\tau_0 \backslash \tau^*) \subset \tau_0$ and supermodularity of G, while the step \eqref{step:stoch_bp_8} is due to the following trivial chain of inequalities:
    \begin{align*}
        \J(\pi^*) &= \EV_{\tau^* \sim \pi^*}\bigg[ F(\tau^*) \bigg]\\
        &\geq \EV_{\tau^*_G \sim \pi^*_G}\bigg[ F(\tau_G^*) \bigg]\\
        &= \EV_{\tau^*_G \sim \pi^*_G}\bigg[ Q(\tau_G^*) \bigg] + \EV_{\tau^*_G \sim \pi^*_G}\bigg[ G(\tau_G^*) \bigg]\\
        &\geq {\J}_G(\pi_G^*) \tag{Q non-negative}
    \end{align*}
    By plugging equation \ref{eq: last_step_BP_stoch} into equation \ref{eq: side_step_BP_stoch} and then equation \ref{eq: side_step_BP_stoch} into equation \ref{eq: middle_step_BP_stoch} we obtain:
    \begin{align}
        \J(\pi_1) &\geq (1-k_Q){\J}_Q(\pi^*) + (1-k^G){\J}_G(\pi^*) - \frac{k^G}{1-k^G}\J(\pi^*) \nonumber\\
        &= (1-k_Q){\J}_Q(\pi^*) + (1-k^G){\J}_G(\pi^*) - \frac{k^G}{1-k^G}\bigg[ {\J}_Q(\pi^*) + {\J}_G(\pi^*)\bigg] \nonumber\\
        &= \frac{(1-k_Q)(1-k^G) - k^G}{1-k^G}{\J}_Q(\pi^*) + \frac{(1-k^G)^2-k^G}{1-k^G}{\J}_Q(\pi^*) \label{eq: final_BP_stoch}
    \end{align}
    and by trivially lower bounding equation \ref{eq: final_BP_stoch} depending on the value of $k_Q$ and $k^G$ leads to the theorem statement.
\end{proof}

\subsection{Auxiliary Lemmas}
\begin{lemma}[Supermodular telescoping bound]
\label{lemma:telescoping_supermodular}
Consider a supermodular function $F$ defined over the ground set $V$. For every set $S \subseteq V$ we have
    \begin{equation*}
        \sum_{j \in S} F(j) \geq (1-k^F)F(S)
    \end{equation*}
\end{lemma}
\begin{proof}
First, we order the elements of the set $S = \{j_1, \ldots, j_n\}$, then we have
    \begin{align*}
        \sum_{j \in S} F(j) &\geq (1-k^F)\big[F(j_1) + F(j_2 \mid j_1) + \ldots + F(j_n \mid j_1, \ldots, l_{n-1} )\big] & \text{curvature inequality}\\
        &= (1-k^F)F(S) & \text{telescoping sum}
    \end{align*}
\end{proof}

%% file: sections/det_case_appendix.tex
Analogously to the analysis presented in section \ref{sec:guarantees}, it is possible to study the approximation guarantees achieved by algorithm \GTO for the trajectory-optimization version of the problem for deterministic GMDPs as presented in equation \eqref{eq:global_reinforcement_learning_traj}. 

For the sake of completeness, in the following we specialize the concept of Non-additive Suboptimality Gap (Definition \ref{def:suboptimality_gap}) in order to capture the sub-optimality gap achieved by a given trajectory in a deterministic GMDP rather than a policy in a possibly stochastic GMDP.

\begin{restatable}[Non-additive Suboptimality Gap: Trajectory-Optimization, Deterministic GMDP Version]{definition}{suboptimalityGapTraj}
\label{def:suboptimality_gap_traj} 
Consider a trajectory $\tau \in \C_\cmp$, where $\C_\cmp$ denotes the CMP constraint \ie the set of admissible trajectories, for a given Global MDP $\cmp_F$. We define the non-additive suboptimality gap $\Rmath(\tau)$ of a trajectory $\tau$ as:
\begin{equation}
    \Rmath(\tau) \coloneqq F^* - F(\tau)
\end{equation}
where $F^* \coloneqq  \max_{\tau \in \C_\cmp} \;\; F(\tau)$ and we implicitly consider an arbitrary initial starting state $s_0$. 
\end{restatable}

Given this definition we can state the following guarantees.

\begin{restatable}[Approximation Guarantee \GTO, $F$ submodular, Deterministic Case]{theorem}{guaranteeSubmodularDeterministic}
\label{theorem:guarantee_submodular_deterministic} 
By running for one iteration algorithm \GTO on a GMDP $\gmdp$ with $F$ submodular, we obtain a trajectory $\tau_1$ such that:
\begin{equation*}
    \Rmath(\tau_1) \leq k_F F^*
\end{equation*}
where $F^* \coloneqq  \max_{\tau \in \C_\cmp} \;\; F(\tau)$.
\end{restatable}
\begin{proof}
We first show that $F(\tau_1) \geq (1-k_F)F(\tau^*)$ where $\tau^* \in \argmax _{\tau \in \C_\cmp} \;\; F(\tau)$. Once this is proved, the theorem statement can be trivially derived by using the definition of $\Rmath(\tau)$.
\setlength{\jot}{7pt}
    \begin{align*}
        F(\tau_1) &\labelrel\geq{step:det_sub_1} m^\sigma_{\tau_0}(\tau_1)\\
        &\labelrel\geq{step:det_sub_2} m^\sigma_{\tau_0}(\tau^*)\\
        &= \sum_{x \in \tau^*} m^\sigma_{\tau_0}(I_{\tau_0}(x))\\
        &= \sum_{x \in \tau^*} F\big(S^{\sigma(\tau_0)}_{I_{\tau_0}(x)}\big) - F\big(S^{\sigma(\tau_0)}_{I_{\tau_0}(x)-1}\big)\\
        &= \sum_{x \in \tau^*} F(\{\sigma(\tau_0)[1], \ldots, \sigma(\tau_0)[I_{\tau_0}(x)]\}) - F(\{\sigma(\tau_0)[1], \ldots, \sigma(\tau_0)[I_{\tau_0}(x)-1]\})\\
        &= \sum_{x \in \tau^*} F(\sigma(\tau_0)[I_{\tau_0}(x)] \mid \{\sigma(\tau_0)[1], \ldots, \sigma(\tau_0)[I_{\tau_0}(x)-1]\})\\
        &\labelrel\geq{step:det_sub_3} (1-k_F) \sum_{x \in \tau^*} F(\sigma(\tau_0)[I_{\tau_0}(x)])\\
        &= (1-k_F)\sum_{x \in \tau^*}F(x)\\
        &\labelrel\geq{step:det_sub_4} (1-k_F)F(\tau^*)\\
    \end{align*}
\end{proof}
We define $I_{\tau}(x)$ as the function returning the index of element $x$ in trajectory $\tau$. In step \eqref{step:det_sub_1} we use the fact that $m^\sigma_{\tau_0}(\tau_1)$ is by construction a lower bound of $F(\tau_1)$, in step \eqref{step:det_sub_2} we notice that $\tau_1$ is a maximizer of $m^\sigma_{\tau_0}(\cdot)$ due to the local optimization step in \GTO, in step \eqref{step:det_sub_3} we leverage the notion of submodular curvature (Definition \ref{def:submodular_curvature}), and in step \eqref{step:det_sub_4} we exploit submodularity of $F$
 (Definition \ref{def:submodularity}).
 
\begin{restatable}[Approximation Guarantee \GTO, $F$ supermodular, Deterministic Case]{theorem}{guaranteeSupermodularDeterministic}
\label{theorem:guarantee_supermodular_deterministic} 
By running for one iteration algorithm \GTO on a GMDP $\gmdp$ with $F$ supermodular, we obtain a trajectory $\tau_1$ such that:
\begin{equation*}
    \Rmath(\tau_1) \leq \frac{2k^F-(k^F)^2}{1-k^F} F^*
\end{equation*}
where $F^* \coloneqq  \max_{\tau \in \C_\cmp} \;\; F(\tau)$.
\end{restatable}

\begin{proof}
We first show that $F(\tau_1) \geq \frac{(k^F)^2 - 3k^F + 1}{1-k^F}F(\tau^*)$ where $\tau^* \in \argmax _{\tau \in \C_\cmp} \;\; F(\tau)$. Once this is proved, the theorem statement can be trivially derived by using the definition of $\Rmath(\tau)$.
\setlength{\jot}{7pt}
    \begin{align*}
       F(\tau_1) &\labelrel\geq{step:det_sup_1} m_{\tau_0}(\tau_1)\\
       &\labelrel\geq{step:det_sup_2} m_{\tau_0}(\tau^*)\\
       &= F(\tau_0) - \sum_{j \in \tau_0 \backslash \tau^*} F(j \mid \tau_0 \backslash j) + \sum_{j \in \tau^* \backslash \tau_0} F(j)\\
       &= F(\tau_0) - \sum_{j \in B} F(j \mid \tau_0 \backslash j) + \sum_{j \in C} F(j)\\
       &\geq \sum_{j \in A \cup B} F(j) - \sum_{j \in B} F(j \mid \tau_0 \backslash j) + \sum_{j \in C} F(j)\\
       &= \sum_{j \in A \cup C} F(j) - \sum_{j \in B} F(j \mid \tau_0 \backslash j) + \sum_{j \in B} F(j)\\
       &\labelrel\geq{step:det_sup_3} \sum_{j \in A \cup C} F(j) - \frac{1}{1-k^F}\sum_{j \in B} F(j) + \sum_{j \in B} F(j)\\
       &\geq \sum_{j \in A \cup C} F(j) - \frac{k^F}{1-k^F}\sum_{j \in B} F(j)\\
       &\labelrel\geq{step:det_sup_4} (1-k^F)F(\tau^*) - \frac{k^F}{1-k^F}\sum_{j \in B} F(j)\\
       &\labelrel\geq{step:det_sup_5}(1-k^F)F(\tau^*) - \frac{k^F}{1-k^F}F(\tau^*)\\
       &= \frac{(k^F)^2 - 3k^F + 1}{1-k^F}F(\tau^*)
    \end{align*}
    where where we have defined $A := \tau_0 \cap \tau^*, B := \tau_0 \backslash \tau^*, C:= \tau^* \backslash \tau_0$. In step \eqref{step:det_sup_1} we use the fact that $m^\sigma_{\tau_0}(\tau_1)$ is by construction a lower bound of $F(\tau_1)$, in step \eqref{step:det_sup_2} we notice that $\tau_1$ is a maximizer of $m_{\tau_0}(\cdot)$ due to the local optimization step in \GTO, in step \eqref{step:det_sup_3} we leverage the notion of supermodular curvature (Definition \ref{def:submodular_curvature}), and in step \eqref{step:det_sup_4} we use Lemma \ref{lemma:telescoping_supermodular}. Meanwhile, step \eqref{step:det_sup_5} is due to the following chain of inequalities:
    \begin{align*}
        \sum_{j \in B}F(j) &\leq F(B) & \tag{supermodularity}\\
        &\leq F(\tau_0) & \tag{monotonicity}\\
        &\leq F(\tau^*) & \tag{optimality of $\tau^*$}\\
    \end{align*}
\end{proof}

\begin{restatable}[Approximation Guarantee \GTO, $F$ BP, Deterministic Case]{theorem}{guaranteeBPDeterministic}
\label{theorem:guarantee_BP_deterministic}
By running for one iteration algorithm \GTO on a GMDP $\gmdp$ with $F = Q + G$ BP, we obtain a trajectory $\tau_1$ such that:
\begin{equation*}
    \Rmath(\tau_1) \leq \alpha F^*
\end{equation*}
with 
\[
    \alpha  =
    \begin{cases}
    \frac{2k^G-(k^G)^2}{1-k^G}  & \text{if $k_F\leq k^G$} \\
    \frac{1-(1-k_Q)(1-k^G)}{1-k^G} & \text{otherwise} 
    \end{cases}
\]
where $k_Q$ is the submodular curvature of $Q$,  $k^G$ is the supermodular curvature of $G$, and $F^* \coloneqq  \max_{\tau \in \C_\cmp} \;\; F(\tau)$. 
\end{restatable}

\begin{proof}
    We define $I_{\tau}(x)$ as the function returning the index of element $x$ in trajectory $\tau$. We first show that $F(\tau_1) \geq (1-\alpha)F(\tau^*)$ where $\tau^* \in \argmax _{\tau \in \C_\cmp} \;\; F(\tau)$. Once this is proved, the theorem statement can be trivially derived by using the definition of $\Rmath(\tau)$.
    \begin{align*}
        F(\tau_1) &= Q(\tau_1) + G(\tau_1)\\
        &\labelrel\geq{step:det_bp_1} m_{\tau_0}^\sigma(\tau_1) + m_{\tau_0}(\tau_1)\\
        &\labelrel\geq{step:det_bp_2} m_{\tau_0}^\sigma(\tau^*) + m_{\tau_0}(\tau^*)\\
        &\labelrel\geq{step:det_bp_3} (1-k_Q)Q(\tau^*)+ m_{\tau_0}(\tau^*)\\
    \end{align*}
    In step \eqref{step:det_bp_1} we use the fact that $m^\sigma_{\tau_0}(\tau_1)$ and $ m_{\tau_0}(\tau_1)$ are by construction respectively lower bounds of $Q(\tau_1)$ and $G(\tau_1)$. In step \eqref{step:det_bp_2}, we notice that $\tau_1$ is a maximizer of the sum of marginal lower bounds, namely $m_{\tau_0}^\sigma(\cdot) + m_{\tau_0}(\cdot)$ due to the local optimization step in \GTO. Meanwhile step \eqref{step:det_bp_3} can be trivially derived by following steps \eqref{step:det_sub_1} to \eqref{step:det_sub_4} used to prove theorem \ref{theorem:guarantee_submodular_deterministic}.
    Now, we define $A := \tau_0 \cap \tau^*, B := \tau_0 \backslash \tau^*, C:= \tau^* \backslash \tau_0$ and lower bound the second term as follows:
    \begin{align}
        m_{\tau_0}(\tau^*) &\labelrel\geq{step:det_bp_4} (1-k^G)G(\tau^*) - \frac{k^G}{1-k^G}\sum_{j \in B}G(j) \nonumber\\
        &\labelrel\geq{step:det_bp_5} (1-k^G)G(\tau^*) - \frac{k^G}{1-k^G}G(\tau_G^*) \nonumber\\
        &\labelrel\geq{step:det_bp_6} (1-k^G)G(\tau^*) - \frac{k^G}{1-k^G}F(\tau_G^*) \label{eq: lower_bound_supermodular_BP}
    \end{align}
    where in step \eqref{step:det_bp_4} we have trivially followed steps \eqref{step:det_sup_2} to \eqref{step:det_sup_4} used to prove theorem \ref{theorem:guarantee_supermodular_deterministic}, while in step \eqref{step:det_bp_5} we have used the following chain of inequalities: $$\sum_{j \in B}G(j) \leq G(B) \leq G(\tau_0) \leq G(\tau_G^*)$$ and in step \eqref{step:det_bp_6} we have used the fact that:
    \begin{equation*}
        F(\tau^*) \geq F(\tau_G^*) = Q(\tau_G^*) + G(\tau_G^*) \geq G(\tau_G^*)
    \end{equation*}
    since $Q$ is non-negative. By plugging equation \ref{eq: lower_bound_supermodular_BP} into the initial chain of inequalities, we obtain:
    \begin{align*}
          F(\tau_1) &\geq (1-k_Q)Q(\tau^*) + (1-k^G)G(\tau^*) - \frac{k^G}{1-k^G}F(\tau^*)\\
          &= (1-k_Q)Q(\tau^*) + (1-k^G)G(\tau^*) - \frac{k^G}{1-k^G}\big[Q(\tau^*)+ G(\tau^*)\big]\\
          &= \frac{(1-k_Q)(1-k^G)-k^G}{1-k^G}Q(\tau^*) + \frac{(1-k^G)^2-k^G}{1-k^G}G(\tau^*)
    \end{align*}
    from which we can straightforwardly derive the statement by lower bounding the coefficients depending on the values of $k_Q, k^G$ and using the definition of $F$. 
\end{proof}

%% file: sections/proof_hardness.tex
\hardness*

\begin{proof}
   The proof is based on a reduction from a problem with a known hardness result, namely \citep[Theorem 4.1]{bai2018greed}, which gives the same approximation ratio as in the lemma, but for the cardinality constrained case. 
   We refer with $P1$ to the problem of BP maximization with a  cardinality constraint, while with $P2$ to the problem of BP maximization under a CMP constraint.
   The reduction works as follows. First, we define a poly-time reduction from any instance of $P1$ to a specific instance of $P2$. In particular, given an instance of $P1$ with a function $F$, a ground set $\V$ and a cardinality constraint $k$, we define an instance of $P2$ with CMP constraint $\mathcal{C}_{\mathcal{M}}$ induced by the fully-connected CMP  $\mathcal{M} := \langle \V, \Aspace, P, H=k\rangle$ as explained in definition \ref{def: cmp_constraint}. The time-extended CMP $\mathcal{M}_H$ will be a $k$-layered CMP of which the state space is a $k$-fold cartesian product of the original ground set $\V$. We define the objective function of $P2$ as $F': D \coloneqq \V \times [k] \to \R$ and $F(S_d) = F(\Pi S_d)$ where $S_d \subseteq D$ and $\Pi: \V \times [k] \to \V$ is a projector map that drops the time-coordinate of its input, e.g. $\Pi (\{(s,t), (s',t')\}) = \{s,s'\}$. For the sake of notational simplicity we write $\Pi S$ instead of $\Pi(S).$
   Notice that the instance of $P2$ can be computed in poly-time \wrt the cardinality of the original ground set $\V$, which represents the complexity of the initial instance of $P1$. In order to show that the instance we have built for $P2$ is a valid one, it is left to show that $F' \in BP$. We start by noticing that
   \begin{equation*}
       F'(S_d) = F(\Pi S_d) = Q(\Pi S_d) + G(\Pi S_d) = Q'(S_d) + G'(S_d)
   \end{equation*}
   where we have used the fact that $F \in BP$ and have defined $Q' := Q\Pi$ and $G' := G\Pi$. Since $Q$ and $G$ are non-negative then also $Q'$ and $G'$ must be non-negative, and since $Q, G, \Pi$ are monotone then also $Q'$ and $G'$ are monotone. Moreover, once can easily check that $Q'$ is submodular.
   Next, we show that $G'$ is supermodular and that it preserves the supermodular curvature of $G$, \ie $k^{G'} = k^{G}$.
   Consider the sets $A_d \subseteq B_d \subseteq D$ and the element $d \notin B_d$. In order to prove that $G'$ is supermodular we must show that $G'(d \mid A_d) \leq G'(d \mid B_d)$. We define $S := \Pi S_d$ and write:
   \begin{align*}
       G'(d \mid B_d) &= G'(B_d \cup d)- G'(B_d)\\
       &= G(\Pi(B_d \cup d)) - G(\Pi B_d)\\
       &= G(\Pi B_d \cup \Pi d) - G(\Pi B_d)\\
       &= G(\Pi d \mid \Pi B_d)\\
       &\geq G(\Pi d \mid \Pi A_d) \tag{$A_d \subseteq B_d \implies \Pi A_d \subseteq \Pi B_d$}\\
       &= G(\Pi A_d \cup \Pi d) - G(\Pi A_d)\\
       &= G(\Pi(A_d \cup d)) - G(\Pi A_d)\\
       &= G'(A_d \cup d) - G'(A_d)\\
       &= G'(d \mid A_d)
   \end{align*}
   which proves that $G'$ is supermodular. As for its curvature, we have:
   \begin{align*}
       k^{G'} &= 1 - \min_{S_d \subseteq D, d \notin S_d} \frac{G'(d)}{G'(d \mid S_d)}\\
       &= 1- \min_{S_d \subseteq D, d \notin S_d} \frac{G(\Pi S_d)}{G(\Pi S_d \cup \Pi d) - G(\Pi S_d)}\\
       &= 1 - \min_{S \subseteq \V, v \notin S} \frac{G(v)}{G(v \mid S)}\\
       &= k^G
   \end{align*}
    For the sake of contradiction, we now suppose that there exists a poly-time algorithm that can solve $P2$ by computing a set $\hat{S}$ such that for every function $F' \in BP$ and $\epsilon > 0$ we have:
    \begin{equation}
        F'(\hat{S}_d) > (1-k^{G'} + \epsilon)F'(S_d^*) \label{eq: ineq_hardness}
    \end{equation}
    where $S_d^*$ is an optimizer of $F'$.
    We claim that eq. \ref{eq: ineq_hardness} implies that $F(\hat{S}_d) > (1-k^{G} + \epsilon)F(S^*)$, where $S^*$ is an optimizer of $F$. Which would be a contradiction with the aforementioned hardness result and would imply the result stated in the lemma.
    In order to prove that
    \begin{equation*}
        F'(\hat{S}_d) > (1-k^{G'} + \epsilon)F'(S_d^*) \implies F(\hat{S}_d) > (1-k^{G} + \epsilon)F(S^*)
    \end{equation*}
    we notice that $F'(S_d) = F(\Pi S_d)$ by definition, and therefore it is left to prove that $F'(S_d^*) = F(S^*)$. By def. of $F'$ we have that $F'(S_d^*) = F(\Pi S_d^*)$, hence it suffices to show that $\Pi S_d^* = S^*$. By contradiction, we suppose that $S^* \neq \Pi S_d^*$. By def. of $S^*$ this would imply that $F(S^*) > F(\Pi S_d^*)$. But notice that $\forall S \subset \V$, the pre-image of $S$ along $\Pi$ is always a non-empty subset of $D$, namely $\Pi^{-1}(S)$ with $\Pi^{-1}: 2^V \to 2^D$. Therefore we can pick $\Bar{S}_d \in \Pi^{-1}(S)$ and we would obtain that:
    \begin{equation*}
        F'(\Bar{S}_d) = F(\Pi \Bar{S}_d) = F(S^*)> F(\Pi S_d^*) = F'(S_d^*)
    \end{equation*}
    which is a contradiction since $S_d^*$ is a maximizer of $F'$ by definition. This fact, together with the fact that $k^{G'} = k^G$ proves our claim. Ultimately, notice that once an optimal solution for $P2$ is computed, an optimal solution for $P1$ can be computed in poly-time. 
\end{proof}

%% file: sections/appx_algorithm.tex
In this section, we first present two propositions, especially for the algorithm \GTO that ensure the build modular functions are tight lower bound of the global reward function, and guarantee monotonic improvement. Then we present the \GPO algorithm and finally discuss some efficient ways to build modular lower bounds.

\begin{proposition}[Tight modular lower bound] \looseness -1 Let $F:2^V \to \R$ be a submodular function. For any set $X \subseteq V$, define permutation $\sigma: [|V|] \to V$ such that $S^\sigma_{|X|}=X$, where $S^\sigma_i = \{\sigma(1), \sigma(2), \hdots, \sigma(i)\}$ and $S^\sigma_0 = \emptyset$. Define a modular function about the set $X$, $m^\sigma_X = \sum_{v \in X} m^{\sigma}_X(v)$ with entries for element $i\in [|V|]$ given by $m^{\sigma}_X(\sigma(i))\coloneqq F( S^\sigma_i) -  F(S^\sigma_{i-1})$. Then $m^\sigma_X$ is a tight modular lower bound of the submodular function F, i.e.,  $m_X(X) = F(X)$ with $m_X(Y) \leq F(Y), \forall Y \subseteq V$.
\label{prop:tight_LB}
\end{proposition}
\begin{proof} As per definition, $m^{\sigma}_X(X) =  \sum_{v \in X} m^{\sigma}_X(v) = \sum_{i\in[|X|]} F( S^\sigma_i) -  F(S^\sigma_{i-1}) = \sum_{i\in[|X|]} F( S^\sigma_i | S^\sigma_{i-1}) = F(X)$. The last equality follows since $S^\sigma_{|X|}=X$. Hence the modular function is tight at $X$. Next, we prove that it is lower bound, i.e., $\forall Y: m^{\sigma}_X(Y) \leq F(Y)$.

Let $Y = \{ i_1, \hdots i_k\}$, wlog, s.t, $i_j <^{\sigma} i_{j+1}$, i.e., the elements are arranged in $Y$ as per permutation $\sigma$.
\begin{align*}
    F(Y) &= \sum_{j=1}^{k} F(i_j | i_1, \hdots i_{j-1}) \tag{$\{i_1, \hdots i_{j-1} \} \subseteq \{ \sigma(1), \hdots \sigma(j-1)\}$}\\
       &\geq \sum_{j=1}^{k} F(i_j | \sigma(1), \hdots \sigma(j-1)) = m^{\sigma} (Y)
\end{align*}
\end{proof}

\begin{proposition}[Monotonic Improvement]\label{lem:monotonic_improvement}
    \GTO monotonically improves the objective function, i.e., at any iteration $t$, it holds that $F(\tau_{t+1}) \geq F(\tau_{t})$.
\end{proposition}
\begin{proof} Let $m_{\tau_{t}}$ be a modular lower bound of a global reward function, F about the trajectory $\tau_t$. Then,
\begin{align*}
    F(\tau_t) &= \sum_{v\in \tau_t} m_{\tau_{t}}(v) \\
    &\labelrel\leq{step:mono_improv} \sum_{v\in \tau_{t+1}} m_{\tau_{t}}(v) \\
    &\labelrel\leq{step:lb} F(\tau_{t+1})
\end{align*}
In the above proof, \eqref{step:mono_improv} follows since $\tau_{t+1}$ is the optimal policy for modular rewards $m_{\tau_t}$ obtained by \mdpSolver, e.g., value iteration. Step \eqref{step:lb} follows since $m_{\tau_{t}}$ is a lower bound of the global reward function $F$.   
\end{proof}

\subsection{Policy optimization for stochastic GMDPs}

In this section, we extend the approach defined for deterministic GMDPs to stochastic  GMDPs. The core idea stays the same, i.e., we recursively approximate the stochastic GMDPs with stochastic linearized MDP and solve it with standard MDP tools. However, we need a mechanism to convert the stochastic GMDPs to linearized  MDPs. What should we linearise it about?

We solve the problem in policy space and linearize the GMDP around the current policy. However, we can compute the lower bounds only around the sets. Policy can thus interpreted as distribution over the trajectories and we define the modular rewards around a policy $\pi$ as:
\begin{align}
    m^E_{\pi} \coloneqq \EV_{\tau \sim \pi} [m_{\tau}] \label{eq: apx_stochastic_lb}
\end{align}
where $m^E_{\pi} \in \R^\V$ defines a modular reward around a policy $\pi$ and $m_{\tau} \in \R^\V$ are reward computed around a trajectory $\tau$. Furthermore we define the linearization of the \GRL objective $J(\pi)$ around the policy $\pi'$, by evaluating this \cref{eq: apx_stochastic_lb} for a policy $\pi$ as,
\begin{align*}
    m^E_{\pi'}(\pi) \coloneqq \EV_{\tau \sim \pi} \EV_{\tau' \sim \pi'} [m_{\tau'}(\tau)]
\end{align*}
\algdef{SE}[DOWHILE]{Do}{doWhile}{\algorithmicdo}[1]{\algorithmicwhile\ #1}%

\begin{algorithm}[!t]
\caption{Global Policy Optimization (\GPO)}
\begin{algorithmic}[1]
\label{alg:global_policy_optim}
\State \textbf{Initialize} GMDP, $\pi_1 \xleftarrow{} random$, $t \leftarrow 0$, 
\Do 
\State $t \leftarrow t+1$
\State Estimate $m^{E}_{\pi_{t}} \coloneqq \EV_{\tau \sim \pi_{t}}[m_{\tau}]$ \label{alg:GPO:linearize_MDP}
\State  $\pi_{t+1} \xleftarrow{} \mdpSolver(\mdp = \langle \Sspace, \Aspace, P, s_0, m^E_{\pi_t}, H \rangle)$ \label{alg:GPO:MDPsolver}
\doWhile{$J(\pi_{t+1})>J(\pi_{t})$} \label{alg:GPO:compare_objective}
\State Return $\pi_t$  \label{alg:GPO:return}
\end{algorithmic}
\label{alg:GPO_algorithm}
\end{algorithm}

\mypar{Global Policy Optimization (\GPO)} The outline of the steps is given in \cref{alg:GPO_algorithm}. The algorithm starts with an arbitrary policy $\pi_1$, which may be represented using e.g., a value function. 
In each iteration, we estimate a modular lower bound of the global reward function about the current policy $\pi$ (\cref{alg:GPO:linearize_MDP}) using Monte Carlo samples. In particular, we sample trajectories utilizing the current policy $\pi_t$ in the \GMDP and compute modular lower bound w.r.t. each of the sampled trajectories. On averaging across all the modular lower bounds, we define a reward for each state-time pair which forms a classic \MDP with modular rewards. Given the modular rewards, we can deploy any \mdpSolver, e.g., value iteration, policy iteration or linear program, etc that solves $\argmax_{\pi} m^{E}_{\pi_{t}}(\pi)$ and results in the optimal policy for the linearized \MDP $\mdp$ in \cref{alg:GPO:MDPsolver}. The algorithm continues and we linearize the \GMDP about the improved stochastic policy and resolve it for a better policy. In every iteration, we compare the objective with the last policy empirically using samples. In case the objective doesn't improve anymore, we terminate in \cref{alg:GPO:return} with the best policy $\pi_t$.

\subsection{Alternative modular lower bounds for submodular rewards}
\label{sec:alternative_LB}

\mypar{Computational complexity of computing lower bounds (LB's)} The computation of LB is $\mathcal{O}(|\mathcal{S} \times\mathcal{T}|)$. Notably, solving a finite horizon MDP \citep[c.f. Chapter 1.2]{agarwal2019reinforcement}, for example, using value iteration has per-iteration complexity of $\mathcal{O}(|\mathcal{S} \times \mathcal{T}|^2 |\mathcal{A}|)$, and thus computing the lower bounds is a non-dominant operation that does not affect the computational complexity of the overall algorithm. The challenge of scaling to a large state space and horizon exists in finite horizon MDP solvers as well and is not exclusive to our approach. However, in order to make the computation of LB more efficient in practice, one can incorporate approaches tailored to the problem such as GPO-S, where the lower bound remains the same for a fixed state across multiple time steps and is empirically faster to compute. We next elaborate more on this.

\mypar{Only state dependent lower bounds} In many applications, submodular rewards are naturally defined on the state space $\mathcal{S}$ rather than joint state time space $\V$. For instance, consider a submodular function $F^\prime: 2^{\mathcal{S} \to \R}, F^\prime(S) \coloneqq |\bigcup_{s\in S} D^s|$. We can build a submodular function $F:2^{\mathcal{S}\times \T} \to \R$ using an operator $A: 2^{\Sspace \times \T} \to 2^{\Sspace}$ that drops the time indices and define $F(\tau) \coloneqq F^\prime(A(\tau))$ \citep[Section 2]{prajapat2023submodular}.

For such functions, we can build a modular lower bound about current 
trajectory $\tau=\{(s_i,i)\}_{i=0}^{H-1}$ with a permutation $\sigma = \{ \tau, \mathcal{V} \backslash \tau\}$ as:

\begin{align}
    m^\sigma(s,t) \coloneqq \begin{cases}
 F(S^{\sigma}_i) - F(S^{\sigma}_{i-1})  & \quad (s,t) \in \tau\\
 \left(F\left(S^{\sigma}_i\right) - F\left(S^{\sigma}_{i-1}\right) \right)/H & \quad  (s,t) \in \tau_r\\
0 &\quad (s,t) \in \tau_v
    \end{cases} \label{eqn:modular_lb}
\end{align}

where $S_i^\sigma = \{ \sigma(1), \sigma(2), . . . , \sigma(i)\}$, $(s,t)=\sigma(i)$, $\tau_r = \{ (s,t) \in \mathcal{V} | (s,\cdot) \not\in \tau\}$ is the set of state-time pair where the state is not yet visited by $\tau$ and $\tau_v = \{ (s,t) \in \mathcal{V} | (s,\cdot) \in \tau, (s,t) \not\in \tau \}$ is the set of the state-time pair where the state is visited but at a different time.

Note that these bounds are computationally more efficient to compute, especially in the case of large horizons. Essentially we compute marginal gain for each state not yet visited and assign this as a reward to that state for any future visit. 

\looseness -1 \mypar{Greedy $\sigma$-permutation for lower bounds} In general, we can pick any permutation $\sigma_\tau = \{ \tau, \V\backslash \tau\}$ randomly as long as first $H$ elements are $\tau$. This results in a valid modular lower bound and hence our theoretical results hold. However empirical performance may vary based on the permutation (c.f. \cref{fig:cover_1_traj}). Here we present a strategy to build these lower bounds greedily. For simplicity, we present it for the state-dependent case explained above.

Define a permutation $\sigma = \{\tau, \tau_r, \tau_v \}$, where $\tau_r \coloneqq \{ (s,t) \in \V | (s,\cdot) \not\in \tau\}$ is the set of state-time pair where the state is not visited earlier and $\tau_v \coloneqq \{ (s,t) \in \V | (s,\cdot) \in \tau, (s,t) \not\in \tau \} $ is the set of state-time pair where the state is visited but at a different time. Ignoring the time dimension, we order the states in the set $\tau_r$ as follows,   
\begin{align*}
   \sigma \coloneqq  \argmax_{s} F'(\{s\} \cup S_i) - F'(S_i),
\end{align*}
where $S_i^\sigma = \{ \sigma(1), \sigma(2), \hdots, \sigma(i)\}$. In this permutation, we randomly shuffle the last $\tau_v$ states, which were visited but at different times. Using this permutation, we define the modular rewards, $m^\sigma(s,t)$, as given in \cref{eqn:modular_lb}.

Note that in \cref{eqn:modular_lb}, we divide modular reward by $H$ for $(s,t) \in \tau_r$ which is a careful choice that ensures $m^\sigma(s,t)$ is a valid lower bound and provides equal weightage (reward) across all times in the modularized MDP for the exploring unvisited state.
In particular, for a LB, $m^\sigma(s,t)$ to be valid requires $F(V) \leq \sum_{(s,t)\in V} m^\sigma(s,t)$ for all sets $V \in \mathcal{V}$. Consider a particular case of permutation $\sigma = \{\tau, V, . . . \}$ where $V = \{(s',2) (s',6), . . . ,(s',H-1) \}$ is a set containing the same state $s'$ for all times and for simplicity let $F(\tau \cup (s',2)) - F(\tau)= F((s',2))$. In this case, if we do not divide by $H$, a valid LB is $m^{\sigma}(s',2)=F((s',\cdot))$ and $m^\sigma(s,t)=0,\forall t\neq 2$, i.e., only visiting $s'$ at horizon $t=2$ will have a reward and zero at other times. However, visiting $s'$ at $t=2$ may not be possible due to MDP constraints and visiting it at any other time is not encouraged by the rewards. Hence, we use the normalized equal reward across all times to enhance exploring the unvisited state at any time.

%% file: sections/experiments_details.tex
\mypar{Computation of Non-Markovian policy}  We compute the optimal non-Markovian policy by solving a linear program (LP). The LP is defined with optimization variables, $\pi(a|s)$, cost as objective \eqref{eq:global_reinforcement_learning}, and constraints that $\pi(a|s)$ is a probability simplex. The cost is defined as an expectation over all the trajectories. Computing all the possible trajectories is computationally exponential in the horizon. Thus, we can compare against it only for deterministic environments with small horizons.

All experiments within Section \ref{sec:experiments} are run on a squared grid with $|\Sspace| = 400$ and action space $\Aspace = \{$left, up, down, right, stay$\}$. Each experiment is conducted over $20$ runs and the empirical standard deviation is shown. 
Here, we first report a table summarizing the configuration shared by most experiments and subsequently we will list the deviations from this configuration for a subset of the experiments.
\begin{table}[H]
\centering
\setlength{\arrayrulewidth}{0.5mm}
\setlength{\tabcolsep}{10pt}
\renewcommand{\arraystretch}{1.5}
\resizebox{8cm}{!}{
\begin{tabular}{ccc} 
  \hline
 Variable & Value  \\ 
  \hline 
env.cov\_module & Matern\\
env.alpha & $0.1$\\
env.beta & $2$\\
env.stochasticity\_degree & $0$ (deterministic), $0.1$ (stochastic)\\
env.unsafety\_penalty & $500$\\
env.n\_traj\_samples & $1$ (deterministic), $20$ (stochastic)\\
\hline
\end{tabular} 
}
\caption{Base experimental configuration.}
\end{table}

\mypar{Bayesian D-Optimal Experimental Design}~~ We have run the experiments with horizon $H=10$ for $6$ iterations of \GTO.

\mypar{Diverse Synergies}~~ We have run the experiments with horizon $H=8$ for $6$ iterations of \GPO.

\mypar{Safe State Coverage}~~ We have run the experiments with horizon $H=20$ for $25$ iterations of \GTO.

In the following, we extend Section \ref{sec:experiments} to showcase the performances of \GTO and \GPO on a wider variety of global reward functions capturing more real-world applications and representative enough to later discuss important insights.

In the following experiments, we consider a squared grid with $|\Sspace| = 100$, with action space $\Aspace = \{left, up, down, right, stay\}$, horizon $H=10$, and show the performance of running \GTO and \GPO for $15$ iterations. Each experiment is conducted over 20 runs and the empirical standard deviation is shown in the following plots. Moreover, the trajectories illustrated in figure \ref{fig:cover_1_traj} are generated using $H=31$ and $35$ iterations of \GTO. The plots for the Safe State Coverage experiment in Figure \ref{fig:safe_coverage_1} have been created with $H=20$, $25$ iterations, and stochasticity degree of $0.05$.

\mypar{States Coverage} 
In \cref{fig:cover_1}, we consider the state-coverage submodular global reward function $F(\tau) \coloneqq |\bigcup_{s\in \tau} D^s|$, with $D^s$ being a disk of size $2\times 2$ containing the agent's current state, and its right, up, and right-up neighbouring states \citep{near_optimal_safe_cov}. Notice that this global reward is fully-curved. A policy maximizing the objective $\J(\pi)$ induced by $F$ will try to explore the state space to maximize coverage according to the application-specific definition of the set $D^s$, which in practice depends on the sensors with which the agent is equipped. 
\begin{figure*}[ht]
\centering
{%
    \includegraphics[width=0.37\textwidth]{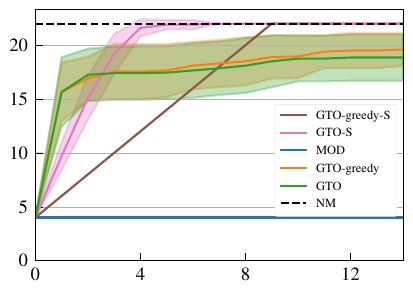}%
}%
\hspace{2cm}
{%
    \includegraphics[width=0.37\textwidth]{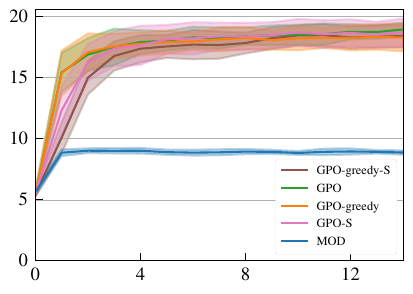}%
}%
\caption[Short Caption]{States Coverage: (left) values of $F(\tau)$ in deterministic GMDP setting where $\tau$ is the trajectory computed by \GTO at each iteration (x-axis), which matches the optimal non-Markovian policy. (right) values of $\J(\pi)$ in stochastic GMDP setting, where $\pi$ is the policy computed by \GPO at each iteration (x-axis).}
\label{fig:cover_1}
\end{figure*}

\begin{figure*}[ht]
\centering
{%
    \includegraphics[width=0.4\textwidth]{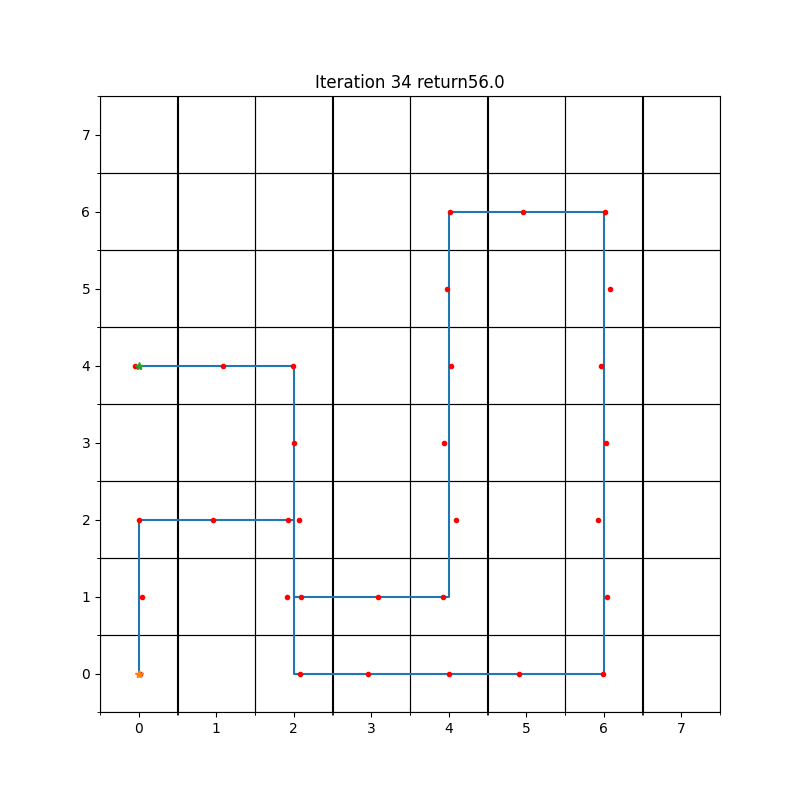}%
}%
\hspace{2cm}
{%
    \includegraphics[width=0.4\textwidth]{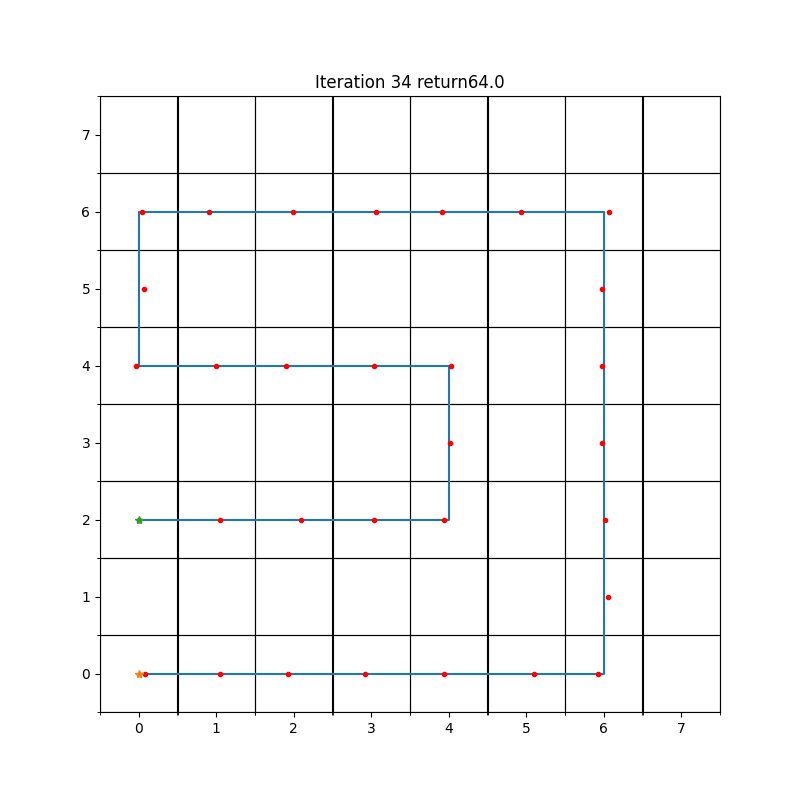}%
}%
\caption[Short Caption]{State Coverage, $H=31, 35$ iterations: (left) trajectory $\tau_1$ induced by output policy of \GTO using GTO-S lower bounds achieves $F(\tau_1) = 56$, (right) trajectory $\tau_2$ induced by output policy of \GTO using GTO-greedy-S lower bounds achieves $F(\tau_2) = 64$. GTO-greedy-S outperforms GTO-S in those instances where the horizon is just enough to reach optimality.}
\label{fig:cover_1_traj}
\end{figure*}

\mypar{Bounded Curvature Coverage} The notion of coverage can often be captured via a bounded-curvature submodular global reward, which we denote as \emph{bounded curvature coverage}. It can be expressed as $F(\tau) = \sum_{s \in \Sspace} \phi(\tau, s)$ where $\phi(\tau, s) = \mathbb{I}_{C(\tau,s)>0}\cdot [1-\alpha(C(\tau,s)-1)]$ and $C(\tau,s) \coloneqq |\{ t \in [H] : (s,t)\in \tau\}|$. Similar to a classic coverage function presented above, this function value is increased by $1$ once a state is visited for the first time, while it increases by an arbitrary value $\alpha$ when a state is visited again. Interestingly, one can prove that the submodular curvature of $F$ has value $k_F = 1-\alpha$. In the following plot, we consider a significantly curved instance, where $\alpha = 0.9$. 

\begin{figure*}[ht]
\centering
{%
    \includegraphics[width=0.37\textwidth]{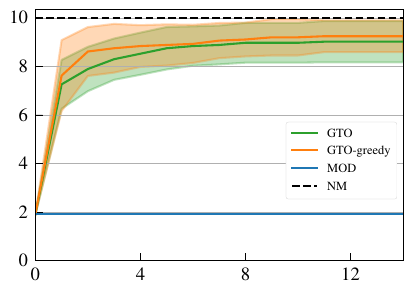}%
}%
\hspace{2cm}
{%
    \includegraphics[width=0.37\textwidth]{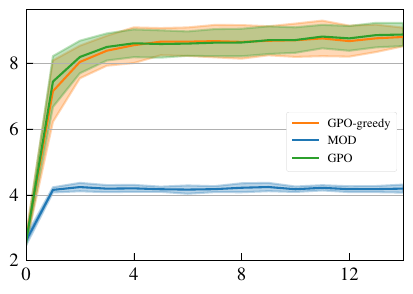}%
}%
\caption[Short Caption]{Bounded Curvature Coverage: (left) values of $F(\tau)$ in deterministic GMDP setting where $\tau$ is the trajectory computed by \GTO at each iteration (x-axis), (right) values of $\J(\pi)$ in stochastic GMDP setting, where $\pi$ is the policy computed by \GPO at each iteration (x-axis). The left plot shows that in practice $\tau$ is nearly-optimal \wrt the optimal non-Markovian policy.}
\label{fig:bounded_curvature_coverage}
\end{figure*}
\mypar{D-Optimal Experimental Design}
Here we consider the optimal experimental design setting as introduced in Section \ref{sec:experiments}.
To ensure robustness in our findings, we conduct 20 experiments across 4 different environments.
\begin{figure*}[ht]
\centering
{%
    \includegraphics[width=0.37\textwidth]{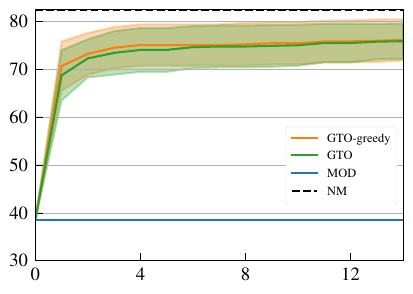}%
}%
\hspace{2cm}
{%
    \includegraphics[width=0.37\textwidth]{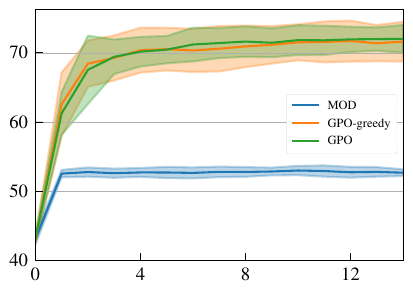}%
}%
\caption[Short Caption]{D-Optimal Experimental Design: (left) values of $F(\tau)$ in deterministic GMDP setting where $\tau$ is the trajectory computed by \GTO at each iteration (x-axis), (right) values of $\J(\pi)$ in stochastic GMDP setting, where $\pi$ is the policy computed by \GPO at each iteration (x-axis). \GTO and \GPO perform nearly optimally in both cases.}
\label{fig:d_experimental_design}
\end{figure*}

\mypar{Synergical Trajectory Selection}
As previously mentioned, in the context of scientific discovery applications, it could be particularly relevant to model positive interactions or synergies among state within a certain trajectory. As an illustrative example, consider states representing atoms and trajectories encoding molecules. Certain combinations \eg pairs, triplets etc., of states \ie atom, can have a synergistic effect that can be captured via supermodular global reward functions. In figure \ref{fig:additive_synergies}, we consider the supermodular global reward function defined as $F(\tau) \coloneqq \sum_{i=1}^K |\tau \cap S_i|^\beta$ with $S_i \subseteq V \coloneqq \Sspace \times T$ indicating a synergy, which we see as a subset of $V$ capturing complementary among its elements.
\begin{figure*}[ht]
\centering
{%
    \includegraphics[width=0.37\textwidth]{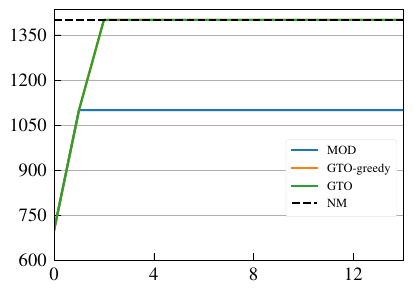}%
}%
\hspace{2cm}
{%
    \includegraphics[width=0.37\textwidth]{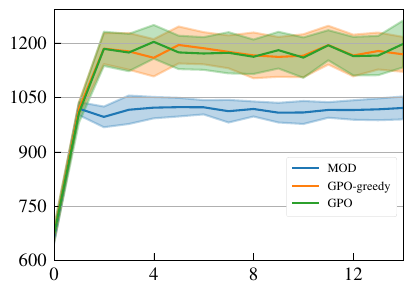}%
}%
\caption[Short Caption]{Synergical Trajectory Selection: (left) values of $F(\tau)$ in deterministic GMDP setting where $\tau$ is the trajectory computed by \GTO at each iteration (x-axis), (right) values of $\J(\pi)$ in stochastic GMDP setting, where $\pi$ is the policy computed by \GPO at each iteration (x-axis). In (left) $\tau$ matches the optimal non-Markovian policy.}
\label{fig:additive_synergies}
\end{figure*}

\mypar{Diverse and Synergical Trajectory Selection}
Interestingly, the notions of exploration mentioned above and encoded through submodularity can be mixed with notions of complementary among states within the same trajectory. This leads to BP objectives such as $F(\tau) = |\bigcup_{s\in \tau} D^s| + \sum_{i=1}^K |\tau \cap S_i|^\beta$ with $S_i \subseteq V \coloneqq \Sspace \times T$. This objective induces policies maximizing state space covering while seeking complementarity between states within the trajectory. The performances of \GTO and \GPO on this global reward are illustrated in figure \ref{fig:diverse_synergies}. We believe that objectives of this type can be particularly relevant in the context of computational chemistry, where often a scientist wishes to discovery chemical compounds that show a certain diversity and complementarity among its elements at the same time.

\begin{figure*}[ht]
\centering
{%
    \includegraphics[width=0.37\textwidth]{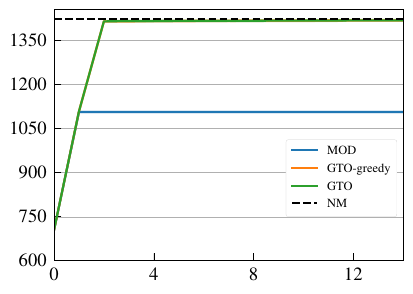}%
}%
\hspace{2cm}
{%
    \includegraphics[width=0.37\textwidth]{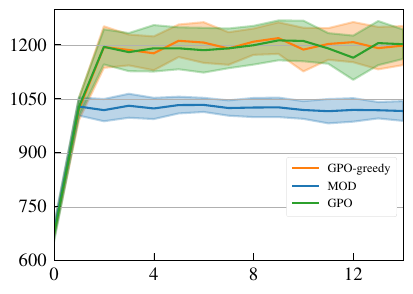}%
}%
\caption[Short Caption]{Diverse and Synergical Trajectory Selection: (left) values of $F(\tau)$ in deterministic GMDP setting where $\tau$ is the trajectory computed by \GTO at each iteration (x-axis), (right) values of $\J(\pi)$ in stochastic GMDP setting, where $\pi$ is the policy computed by \GPO at each iteration (x-axis). From (left) we can deduce that $\tau$ can properly trade-off diversity and complementary and match the optimal non-Markovian policy.}
\label{fig:diverse_synergies}
\end{figure*}

\mypar{Safe States Coverage} Here we consider the notion of safe state coverage as introduced in Section \ref{sec:experiments}. As illustrated in figure \ref{fig:safe_coverage_2}, an optimal policy \wrt to this objective is highly explorative while avoiding unsafe areas. Notice that the concept of safety is captured via a penalty term which can be arbitrary calibrated \wrt the maximum value of the submodular component. Nonetheless, in order to guarantee the satisfiability of a safety constraint one would have to express the global reward as a Lagrangian and compute the optimal Lagrangian multiplier by an outer optimization scheme. This procedure, which we leave as future work, seems particularly viable and may lead to high probability guarantees on safety satisfiability.
\begin{figure*}[ht]
\centering
{%
    \includegraphics[width=0.37\textwidth]{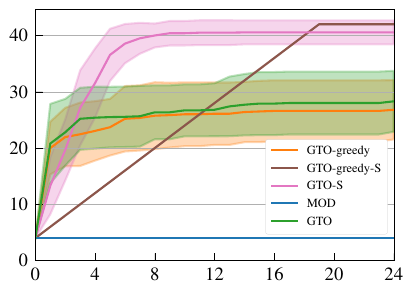}%
}%
\hspace{2cm}
{%
    \includegraphics[width=0.37\textwidth]{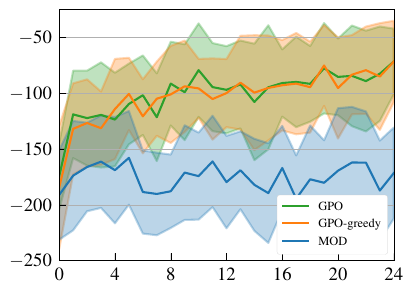}%
}%
\caption[Short Caption]{Safe States Coverage: (left) values of $F(\tau)$ in deterministic GMDP setting where $\tau$ is the trajectory computed by \GTO at each iteration (x-axis), (right) values of $\J(\pi)$ in stochastic GMDP setting, where $\pi$ is the policy computed by \GPO at each iteration (x-axis). Negative values in (right) are due to high unsafety penalty and unavoidable possibility of visiting unsafe states, see figure \ref{fig:safe_coverage_1}.}
\label{fig:safe_coverage_1}
\end{figure*}

\begin{figure*}[ht]
\centering
{%
    \includegraphics[width=0.4\textwidth]{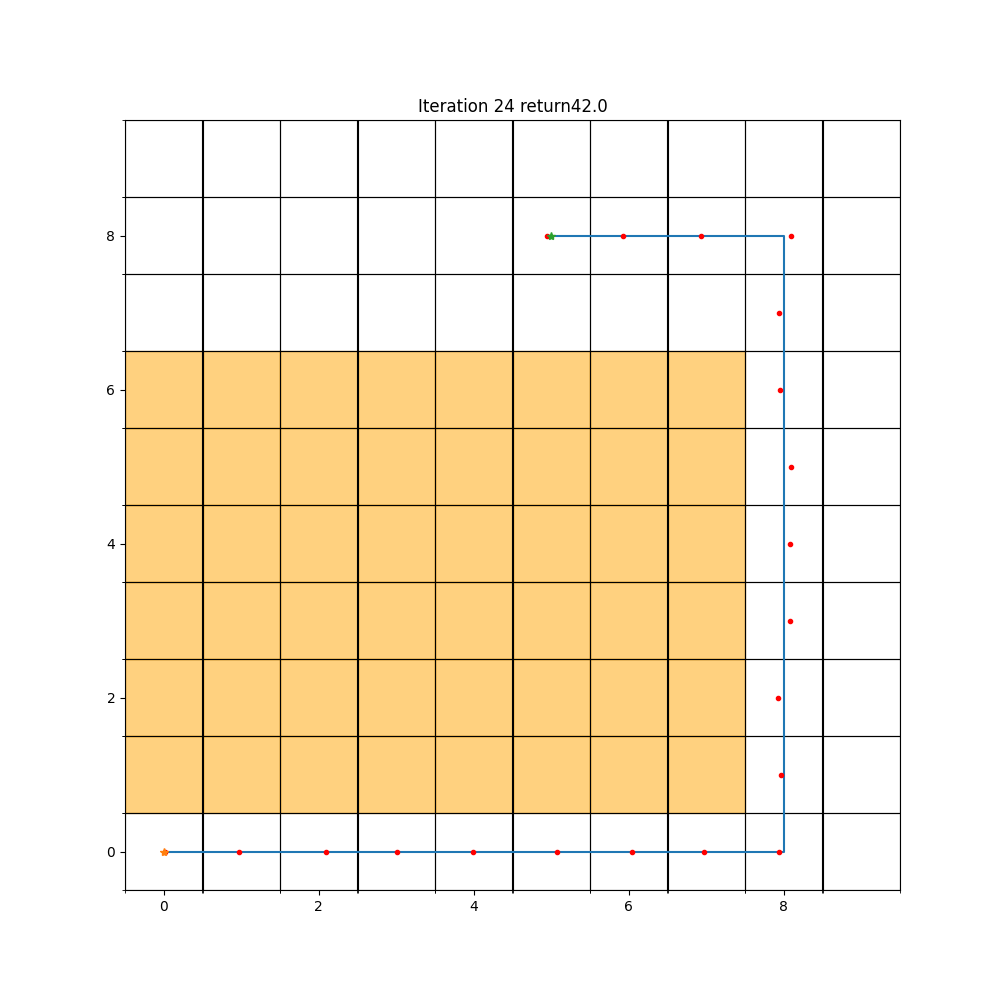}%
}%
\hspace{2cm}
{%
    \includegraphics[width=0.4\textwidth]{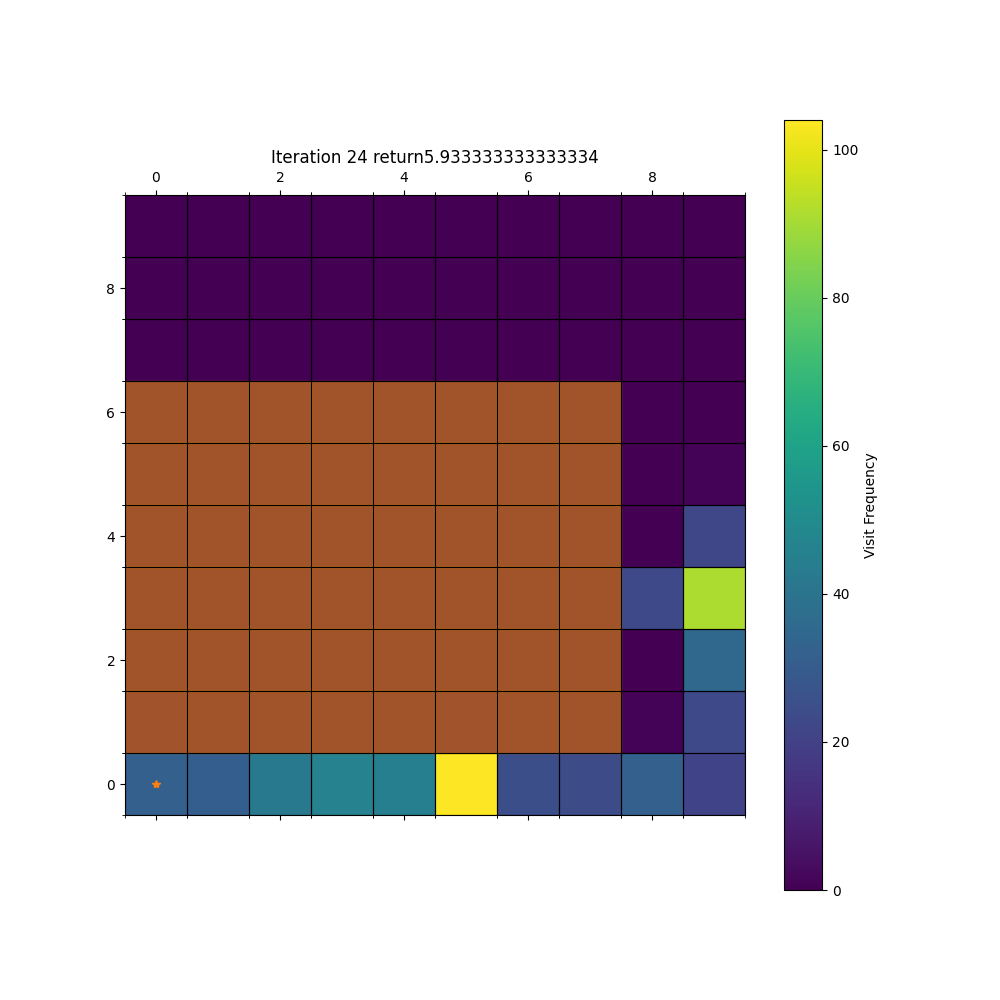}%
}%
\caption[Short Caption]{Safe States Coverage $H=20, 25$ iterations: (left) trajectory $\tau$  computed by \GTO at last iteration, (right) empirical distribution (blue = low probability, yellow = high probability) over the state space induced by the policy $\pi$ computed by \GPO at each iteration (x-axis). The initial state is the bottom left state of the grid. }
\label{fig:safe_coverage_2}
\end{figure*}